\documentclass{article}

\PassOptionsToPackage{round,authoryear}{natbib}

\usepackage[preprint]{neurips_2026}

\usepackage[utf8]{inputenc}
\usepackage[T1]{fontenc}
\usepackage{hyperref}
\usepackage{url}
\usepackage{booktabs}
\usepackage{amsfonts}
\usepackage{nicefrac}
\usepackage{microtype}
\usepackage{xcolor}
\usepackage{algorithm}
\usepackage{algorithmic}

\usepackage[british]{babel}
\usepackage{enumitem}
\usepackage{float}
\usepackage{dblfloatfix}
\usepackage{graphicx}
\usepackage{amsmath,amssymb,mathtools,bm,etoolbox,amsthm}
\usepackage{caption}
\usepackage{subcaption}
\usepackage{wrapfig}
\usepackage[title]{appendix}
\usepackage{titletoc}
\usepackage{placeins}
\usepackage{commath}
\usepackage{xspace}
\usepackage{listings}
\usepackage{tikz}
\usetikzlibrary{fit,positioning}
\usepackage{pgfplots}
\pgfplotsset{compat=1.18}

\graphicspath{{Figures/}{Figures/frame/}}

\definecolor{blcolor}{rgb}{0,0,0.65}
\hypersetup{
  colorlinks,
  linkcolor={blcolor},
  citecolor={blcolor},
  urlcolor={blcolor}
}

\newcommand{\remove}[1]{}

\newcommand{\bfv}[1]{\bm{#1}}

\newcommand{\expect}[1]{\mathbb{E}\!\left[#1\right]}
\newcommand{\expectw}[2]{\mathbb{E}_{#1}\left[#2\right]}

\newcommand{\defeq}{\coloneqq}

\newcommand{\fig}[1]{Fig.~\ref{#1}}

\newcommand{\sect}[1]{Sec.~\ref{#1}}

\newcommand{\app}[1]{App.~\ref{#1}}
\newcommand{\eps}{\varepsilon}
\newcommand{\parag}[1]{{\bf #1}~~}

\theoremstyle{plain}
\newtheorem{theorem}{Theorem}[section]
\newtheorem{proposition}[theorem]{Proposition}
\newtheorem{lemma}[theorem]{Lemma}

\theoremstyle{definition}

\theoremstyle{remark}

\newcommand{\maxpo}{\ensuremath{\mathrm{MaxPO}}}

\usepackage[most]{tcolorbox}

\definecolor{codegreen}{rgb}{0,0.6,0}
\definecolor{codegray}{rgb}{0.5,0.5,0.5}
\definecolor{codepurple}{rgb}{0.58,0,0.82}
\definecolor{backcolour}{rgb}{0.95,0.95,0.92}

\definecolor{softred}{RGB}{220, 80, 80}
\definecolor{softgreen}{RGB}{60, 150, 80}
\definecolor{softorange}{RGB}{230, 140, 60}
\definecolor{softblue}{RGB}{70, 130, 200}
\definecolor{boxgray}{RGB}{95,95,95}

\lstdefinestyle{mystyle}{
    backgroundcolor=\color{backcolour},
    commentstyle=\color{codegreen},
    keywordstyle=\color{magenta},
    numberstyle=\tiny\color{codegray},
    stringstyle=\color{codepurple},
    basicstyle=\ttfamily\footnotesize,     breakatwhitespace=false,
    breaklines=true,
    captionpos=b,
    keepspaces=true,
    numbers=left,
    numbersep=5pt,
    showspaces=false,
    showstringspaces=false,
    showtabs=false,
    tabsize=4,
    language=Python
}

\newtcblisting{examplebox}[2][]{  breakable,
  colback=white,
  colframe=gray!75!black,
  title=#2,
  listing only,
  listing options={
    basicstyle=\normalfont\small,
    columns=fullflexible,
    keepspaces=true,
    showstringspaces=false,
    breaklines=true,
    language={},
    keywordstyle=,
    commentstyle=,
    stringstyle=,
    backgroundcolor=\color{white},
    frame=none,
    rulecolor=\color{white},
    framerule=0pt,
    framesep=0pt,
    xleftmargin=0pt,
    xrightmargin=0pt,
    aboveskip=0pt,
    belowskip=0pt
  },
  #1
}

\setcitestyle{square,numbers}

\title{OrderGrad: Optimizing Beyond the Mean with Order-Statistic Policy Gradient Estimation}

\author{\begin{tabular}{@{}c@{}}
Paavo Parmas \quad Yongmin Kim \quad Kohsei Matsutani \quad Shota Takashiro\\[0.3em]
Soichiro Nishimori \quad Takeshi Kojima \quad Yusuke Iwasawa \quad Yutaka Matsuo\\[0.7em]
\normalfont{The University of Tokyo}\\
\texttt{paavo.parmas@weblab.t.u-tokyo.ac.jp}
\end{tabular}}

\newcommand{\R}{\mathbb{R}}
\newcommand{\E}{\mathbb{E}}
\newcommand{\Prob}{\mathbb{P}}
\newcommand{\ind}{\mathbf{1}}

\newcommand{\estmr}[1]{\hat{\bfv{#1}}}
\newcommand{\gtrue}{\bfv{g}}
\newcommand{\Jmean}{J_{\mathrm{mean}}}

\begin{document}

\maketitle

\begin{abstract}
Policy-gradient methods usually optimize expected return, but many real world applications care about distributional properties of returns: tail risk, outlier robustness, or best-of-K discovery. We introduce OrderGrad, a family of likelihood-ratio and reparameterization gradient estimators for order-statistic objectives. OrderGrad optimizes finite-sample L-statistics, i.e., weighted averages of sorted rewards or costs, recovering objectives such as VaR, CVaR, trimmed means, medians, and top-m/best-of-K criteria by changing only the rank weights. For any fixed sample size and rank-weight vector, OrderGrad provides an unbiased gradient estimator for the corresponding order-statistic objective. The method is implemented as a simple reward transformation that can then be used in an otherwise standard policy-gradient or reparameterized update. We study the resulting estimator’s variance behavior and evaluate it on tasks where mean optimization is mismatched to the deployment objective, including LLM math post-training and other tasks. OrderGrad provides a unified, plug-and-play route to risk-averse, robust, and exploratory learning.\\[0.8ex]
Code: \url{https://github.com/paavo5/ordergrad}
\end{abstract}

\section{Introduction}
\label{sec:introduction}

Most stochastic learning algorithms optimize using a gradient estimator $\estmr{g}$ that is unbiased for the mean objective
$\expect{\estmr{g}} = \nabla_\theta \expectw{x\sim p_\theta}{R(x)}$.  In policy-gradient methods, likelihood-ratio (LR) estimators differentiate log probabilities and weight them by scalar returns or advantages \citep{williams_1992_reinforce}; in reparameterized (RP) models, pathwise estimators differentiate samples written as transformations of parameter-free noise \citep{kingma_welling_2014_autoencoding,rezende_2014_stochastic}.
In practice, likelihood-ratio policy gradients have become a workhorse in high-impact applications, including LLM post-training with human or rule-based feedback \citep{ziegler_2019_finetuning,stiennon2020learning,ouyang_2022_training,ahmadian_2024_back,shao2024deepseekmath,guo_2025_deepseekr1} and robot policy learning for motor control and dexterous manipulation \citep{peters_2006_policy,peters_2008_reinforcement,andrychowicz_2020_dexterous,akkaya_2019_rubiks}.
Both viewpoints are now standard in Monte Carlo gradient estimation \citep{mohamed_2020_mc_gradient,parmas_2021_unified}.  OrderGrad asks how these familiar estimators should change when the objective is not the mean reward but a functional of the whole reward distribution (applicable to both LR and RP as well as regular minibatch gradients).

\parag{Why go beyond the mean?}
The usual criterion $\Jmean(\theta)=\E[R]$ compresses the reward distribution to a single number.  In robust learning, risk-sensitive RL, evaluation of stochastic agents, and preference optimization, this mean can be the wrong target.
An inference-time generation system may sample several completions and select one using a reward model \citep{stiennon2020learning,nakano2021webgpt}; a safety-critical controller may care about lower-tail or percentile risk \citep{chow2015risk,chow2018risk}; a robust learner may trim high-loss corrupted samples \citep{shen2019learning,lugosi2021robust}; and an alignment or preference-learning pipeline may target specific reward quantiles or top-$M$ rankings rather than the mean \citep{wang2025beyond,cai2025k}.
Order statistics are a standard way to describe such distributional goals: they expose lower tails, upper tails, medians, trimmed means, and best-of-$k$ behavior using only sorted outcomes.  More generally, weighted averages of order statistics, or L-statistics, provide a classical and flexible language for robust and distribution-aware objectives \citep{daniell_1920_order,bickel_lehmann_1975_location,huber_ronchetti_2009_robust,maurer_2020_lstat}.

\begin{figure}[t]
    \centering
    \includegraphics[width=1.0\textwidth]{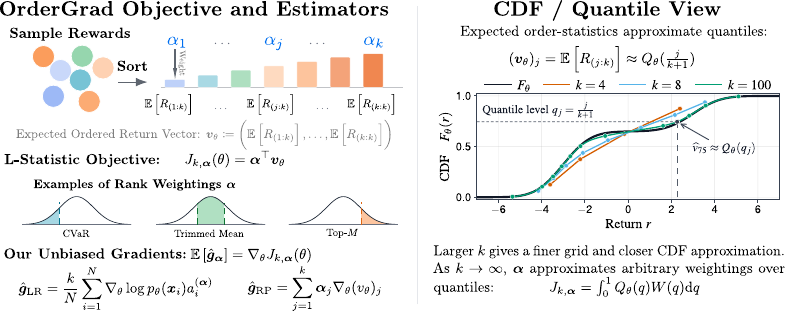}
                \caption{OrderGrad overview.  Rank weights define a distributional objective over sorted rewards.  For likelihood-ratio gradients, OrderGrad transforms rewards into rank advantages $\{R_i\}\mapsto\{a_i^{(\alpha)}\}$; for reparameterization gradients, it differentiates the rank-weighted value $v^\alpha$.  The right side shows the equivalent quantile view of the finite-$k$ objective, shown for a bimodal reward distribution $p(r)$.}
    \label{fig:framework_overview}
\end{figure}

\parag{Order statistics as quantile weighting.}
OrderGrad starts from a simple idea: instead of summarizing a reward distribution only by its mean, we sort a small group of sampled rewards and decide which ranks matter.  Given $k$ samples, write their sorted rewards as
\[
R_{(1:k)} \le \cdots \le R_{(k:k)}.
\]
The lowest ranks describe the lower tail, the middle ranks describe typical outcomes, and the highest ranks describe best-of-$k$ behavior.  As $k$ grows, these sorted rewards give an increasingly fine approximation to the reward distribution: the $j$th order statistic corresponds roughly to the $j/(k+1)$ quantile, and in the limit this view recovers the underlying CDF/quantile curve (\fig{fig:framework_overview}, right).

Thus, choosing rank weights $\alpha_1,\ldots,\alpha_k$ is a way of choosing which parts of the reward distribution to optimize.  Placing weight on large ranks emphasizes high-quantile or best-of-$k$ performance; placing weight on small ranks emphasizes robustness or safety in the lower tail; spreading weight over a range gives trimmed or tail-averaged objectives. The resulting objective is
\begin{equation}
\label{eq:order-obj}
J_{k,\alpha}(\theta)
\defeq
\E\!\left[\sum_{j=1}^k \alpha_j R_{(j:k)}\right] = 
\sum_{j=1}^k \alpha_j\E\!\left[ R_{(j:k)}\right].
\end{equation}

\paragraph{OrderGrad as a simple transformation.}
OrderGrad derives unbiased estimators for \eqref{eq:order-obj} by changing only the scalar learning signal used by familiar gradient estimators.  In likelihood-ratio form, the minibatch rewards are transformed into rank advantages, with one line of code,
\[
\{R_i\}_{i=1}^N \longmapsto \{a_i^{(\alpha)}\}_{i=1}^N, \quad\mathtt{ ~a = ordergrad.lstat\_advantage(rewards, ~alpha)}
\]
which are then plugged into an otherwise standard policy-gradient update. The reparameterization
 form has a similar transform. The explicit LR and RP estimator equations are summarized in \fig{fig:framework_overview} and written in 
 \eqref{eq:og-batch-estimator}; the appendix gives the unbiasedness derivations.  The order-statistic size $k$ defines the target criterion and can be smaller than the batch size $N$. For fixed $N$, $k$, and $\alpha$, the required rank weights can be precomputed and collapsed once, so each minibatch update is sorting plus a linear pass, with total time $O(N\log N)$ dominated by the sort (i.e., negligible in practice).

From the technical perspective, the closest works are the recent line of approaches for unbiased pass@k/max@k optimization 
\cite{koyamada2022remax,nishimori2026emergence,tang2025optimizing,walder_karkhanis_2025_pkpo,chen2025pass,bagirov2025best}, which are a special
case of order-statistics (specifically, they consider only the top rank).
In this line of work, the key question has been how to reduce gradient estimation variance. \citet{koyamada2022remax} initially gave a rudimentary 
LR based estimation method with a baseline computed from separate batches; \citet{tang2025optimizing} proposed a better baseline that uses the max
of the other $k-1$ samples; \citet{walder_karkhanis_2025_pkpo} extended this approach to cases when $N>k$ to achieve better variance reduction; and 
\citet{chen2025pass} and \citet{bagirov2025best} give yet other baselines based on mean estimates for pass@k and max@k respectively.
Our work could be seen as a generalization of these methods to arbitrary order-statistics beyond only the top-1 rank. We give a different derivation
for the baseline computation that allows generalization to our scenario, but matches \citet{chen2025pass} and \citet{bagirov2025best} in the max@k
setting.

Experimentally, we evaluate OrderGrad on LLM math post-training tasks in \sect{sec:reasoning_experiments}.  We use Top-$M$@$K$ objectives, which average the top $M$ members of each $K$-sample group, and show improved pass@$k$ compared with the Max@$K$ special case.  We also combine a Top-$M$@$K$ solve reward with a Bottom-$M$@$K$ length-penalty cost.  This targeted rank weighting significantly shortens outputs without substantially reducing solve rate, whereas a naive GRPO scalarization that simply sums the solve and length rewards gives poor performance.

\section{Preliminaries}
\label{sec:preliminaries}

\subsection{Notation and Monte Carlo gradient estimators}
\label{sec:notation-gradient-estimators}

Let $\theta\in\R^d$ denote the parameters of a sampler, policy, model, or stochastic computation graph.  A random object $x\sim p_\theta$ may be a trajectory in RL, a completion from a language model, an action in a bandit, or a latent variable in a differentiable model.  A scalar reward is denoted by $R(x)\in\R$; larger values are better.  For loss minimization, we set $R=-L$.

For a generic scalar learning criterion $J(\theta)$, write
\begin{equation}
\gtrue(\theta)\defeq \nabla_\theta J(\theta).
\end{equation}
Following common Monte Carlo gradient-estimation notation, an estimator $\estmr{g}$ is a random vector.  It is unbiased for $J$ when
\begin{equation}
\label{eq:unbiased-gradient-estimator}
\E[\estmr{g}]=\gtrue(\theta)=\nabla_\theta J(\theta).
\end{equation}
With $N$ independent samples, the usual batch estimator is $\estmr{g}_N=N^{-1}\sum_{i=1}^N\estmr{g}_i$.

\subsection{The mean objective and elementary estimators}
\label{sec:elementary-estimators}

The standard expected-reward objective is
\begin{equation}
\label{eq:mean-objective}
\Jmean(\theta)=\E_{x\sim p_\theta}[R(x)].
\end{equation}
The likelihood-ratio, or score-function, estimator is
\begin{equation}
\label{eq:lr-prelim}
\estmr{g}_{\mathrm{LR}}
=
\nabla_\theta\log p_\theta(x)\bigl(R(x)-b\bigr),
\qquad
\E[\estmr{g}_{\mathrm{LR}}]=\nabla_\theta\Jmean(\theta),
\end{equation}
where $b$ is a baseline independent of the current sample $x$; more generally, $b$ may depend on the state, prompt, context, or other samples, but not on the sampled action or trajectory whose score is being multiplied.  Under this condition $\E[\nabla_\theta\log p_\theta(x)b]=0$, so the estimator remains unbiased.  This includes population baselines and leave-one-out baselines that exclude the current sample, as in multi-sample Monte Carlo estimators \citep{mnih_rezende_2016_vimco,parmas_2018_pipps}.

When the sample can be written as a differentiable transformation $x=T_\theta(\eps)$ with $\eps\sim p(\eps)$ independent of $\theta$, the reparameterization/pathwise estimator is
\begin{equation}
\label{eq:rp-prelim}
\estmr{g}_{\mathrm{RP}}
=
\nabla_\theta R\bigl(T_\theta(\eps)\bigr)
=
\nabla_x R(x)\,\nabla_\theta T_\theta(\eps),
\end{equation}
assuming $R$ has no additional direct dependence on $\theta$.  LR estimators require only rewards and log-probability gradients, while RP estimators can have lower variance but require a differentiable sampling path.  OrderGrad can be layered on top of either estimator family: in LR form it replaces $R-b$ by a rank-based batch advantage; in RP form it differentiates the batch value $v^\alpha$.

\subsection{Order statistics and L-statistics}
\label{sec:order-stat-prelim}

To move beyond the mean objective, OrderGrad draws $k$ rewards, sorts them as $R_{(1:k)}\le\cdots\le R_{(k:k)}$, and optimizes the rank-weighted criterion $J_{k,\alpha}(\theta)=\E\!\left[\sum_{j=1}^k\alpha_j R_{(j:k)}\right]$.  An L-statistic is a linear combination of order statistics.  In this paper, the order-statistic sample size $k$ is a design choice that determines which distributional property the gradient estimator targets.  The weight vector $\alpha$ can be normalized to sum to one, but this is not required: unnormalized weights are useful when scaling gradients or combining multiple criteria.

A simple finite-population computation will be used repeatedly below.  Suppose a fixed batch has sorted rewards
\[
R_{(1:N)}\le \cdots \le R_{(N:N)}
\]
and let $S$ be a uniformly random size-$k$ subset of the $N$ indices.  The probability that the $m$th sorted batch reward becomes the $j$th order statistic of the subset is
\begin{equation}
\label{eq:prelim-finite-pop-weight}
\omega_{m,j}^{(N,k)}
=
\Prob\!\left((R_S)_{(j:k)}=R_{(m:N)}\mid R_{1:N}\right)
=
\frac{\binom{m-1}{j-1}\binom{N-m}{k-j}}{\binom{N}{k}},
\end{equation}
with the convention that out-of-range binomial coefficients are zero.  This is the standard order-statistic distribution for a simple random sample without replacement from a finite population \citep[Sec.~3.7]{arnold_balakrishnan_nagaraja_1992_first_course}; see also \citet{oneill_2025_srswor_order_statistics}.  Consequently the batch value of the $j$th order statistic can be computed as the fixed weighted sum
\begin{equation}
\label{eq:prelim-vj-fixed-weights}
v_j
=
\E\!\left[(R_S)_{(j:k)}\mid R_{1:N}\right]
=
\sum_{m=1}^N R_{(m:N)}\,\omega_{m,j}^{(N,k)}.
\end{equation}
This counting idea is the basis for the OrderGrad estimator; below we extend it to obtain the include-one value and a principled leave-one-out baseline for the LR gradient estimator.

\subsection{Connection to CDFs, quantiles, and spectral criteria}
\label{sec:cdf-quantile-connection}

For a continuous reward distribution with CDF $F$ and quantile function $Q(u)=F^{-1}(u)$, the probability integral transform lets us sample rewards by first drawing $U_i\sim\operatorname{Unif}(0,1)$ and then setting $R_i\defeq Q(U_i)$.  Since $Q$ is nondecreasing, ranking can be carried out in the uniform sample space and then mapped back through $Q$:
\[
U_{(1:k)}\le\cdots\le U_{(k:k)},
\qquad
R_{(j:k)}\defeq Q(U_{(j:k)}).
\]
The $j$th order statistic of $k$ independent uniforms has distribution $U_{(j:k)}\sim \operatorname{Beta}(j,k+1-j)$ \citep{arnold_balakrishnan_nagaraja_1992_first_course}; write its density as $b_{j,k}$.  This Beta density is a smooth kernel on $[0,1]$ with mean $j/(k+1)$.  Therefore
\begin{equation}
\label{eq:beta-kernel-objective}
\E[R_{(j:k)}]
=\int_0^1 Q(u) b_{j,k}(u)\,du,
\qquad
J_{k,\alpha}(\theta)
=\int_0^1 Q_\theta(u) w_{k,\alpha}(u)\,du,
\end{equation}
where $w_{k,\alpha}(u)=\sum_{j=1}^k \alpha_j b_{j,k}(u)$.  Thus finite-$k$ order objectives are beta-kernel-smoothed quantile-weighted objectives, with $\alpha$ mixing Beta kernels over quantile levels.  As $k$ grows, $J_{k,\alpha}$ approaches arbitrary quantile-weighted criteria $\int_0^1 Q_\theta(u)W(u)\,du$, including CVaR/AVaR and spectral risk measures \citep{acerbi_tasche_2002_expected,acerbi_tasche_2002_coherence,shapiro_2013_kusuoka,holland_haress_2022_spectral}.

\section{OrderGrad method: Unbiased Order-statistic Gradient Estimation}
\label{sec:method}

\subsection{From population advantages to batch advantages}
\label{sec:method-two-regimes}

We first derive the likelihood-ratio form in the ideal population setting.  Draw
\[
A_1,\ldots,A_k\stackrel{i.i.d.}{\sim}p_\theta,
\qquad
R_\ell=R(A_\ell),
\qquad
R_{(1:k)}\le\cdots\le R_{(k:k)}.
\]
For a single rank $j$, define the population order-statistic value and the conditional value of fixing the first draw to action $b$:
\[
\bar v_j(\theta)\defeq \E[R_{(j:k)}],
\qquad
\bar q_{b,j}\defeq \E[R_{(j:k)}\mid A_1=b].
\]
Analogously to \citet{nishimori2026emergence}, applying the REINFORCE identity to the $k$ i.i.d. draws and using exchangeability gives
\begin{align}
\label{eq:method-pop-rank-lr}
\nabla_\theta \bar v_j(\theta)
&=
\E\!\left[R_{(j:k)}\sum_{\ell=1}^k\nabla_\theta\log p_\theta(A_\ell)\right]
=
k\,\E\!\left[\nabla_\theta\log p_\theta(A_1)R_{(j:k)}\right]\nonumber\\
&=
k\,\E_{b\sim p_\theta}\!\left[\nabla_\theta\log p_\theta(b)\,\E[R_{(j:k)}\mid A_1=b]\right]
=
k\,\E_{b\sim p_\theta}\!\bigl[\nabla_\theta\log p_\theta(b)\,\bar q_{b,j}\bigr].
\end{align}
The unconditional value $\bar v_j=\E_{b\sim p_\theta}[\bar q_{b,j}]$ is a natural baseline for the conditional value.  In the sampling estimator below, we use this idea through a leave-one-out baseline that does not require access to $\bar v_j$.  For L-statistic weights $\alpha\in\R^k$, define
\[
\bar v^\alpha\defeq\sum_{j=1}^k\alpha_j\bar v_j,
\qquad
\bar q_b^{(\alpha)}\defeq\sum_{j=1}^k\alpha_j\bar q_{b,j}.
\]

When the distribution and reward table are known, the barred values and gradients can be computed analytically; the binomial-tail formulas are given in \app{app:known-distribution-details}.  In ordinary minibatch optimization, however, we do not have access to $\bar q_{b,j}$ or $\bar v_j$.  The rest of the main method therefore constructs sampling-based analogues from a realized batch: the value $v_j$, the include-one value $q_{i,j}$, the leave-one-out baseline $v_j^{(-i)}$, and the advantage $a_{i,j}=q_{i,j}-v_j^{(-i)}$.

The next subsections derive these batch quantities and then plug them into the two unbiased sample-based estimators
\begin{equation}
\label{eq:og-batch-estimator}
\estmr{g}_{\mathrm{LR\mbox{-}OG}}^\alpha
=
\frac{k}{N}\sum_{i=1}^N a_i^{(\alpha)}\,\nabla_\theta\log p_\theta(x_i),
\qquad
\estmr{g}_{\mathrm{RP\mbox{-}OG}}^\alpha
=
\nabla_\theta v^\alpha(\theta)
=
\nabla_\theta\sum_{j=1}^k\alpha_j v_j(\theta).
\end{equation}
The LR estimator uses the leave-one-out-centered batch advantage.  A standard interchange of differentiation and expectation gives the RP route: estimate the rank-weighted value by $v^\alpha$ on the batch and differentiate it directly.

\subsection{Batch values, include-one values, and advantages}
\label{sec:method-qva-high-level}

This subsection defines the sample-based estimates used in \eqref{eq:og-batch-estimator}. Fix a realized batch of rewards $R_{1:N}$.  For the value estimator we require $1\le k\le N$; for the LR leave-one-out advantage below we require $1\le k<N$, so that a size-$k$ subset can still be drawn after removing one datapoint.  For any size-$k$ subset $S\subseteq[N]$, let $(R_S)_{(j:k)}$ denote the $j$th smallest reward in that subset.  The batch value for rank $j$ is the average over all such subsets,
\begin{equation}
\label{eq:method-vj-def}
v_j
\defeq
\frac{1}{\binom{N}{k}}
\sum_{\substack{S\subseteq[N]\\ |S|=k}}
(R_S)_{(j:k)}.
\end{equation}
For datapoint $i$, the include-one value averages over all size-$k$ subsets that contain $i$,
\begin{equation}
\label{eq:method-qij-def}
q_{i,j}
\defeq
\frac{1}{\binom{N-1}{k-1}}
\sum_{\substack{S\subseteq[N]\\ |S|=k,\ i\in S}}
(R_S)_{(j:k)},
\end{equation}
and the leave-one-out value averages over all size-$k$ subsets after removing $i$,
\begin{equation}
\label{eq:method-vloo-def}
v^{(-i)}_j
\defeq
\frac{1}{\binom{N-1}{k}}
\sum_{\substack{S\subseteq[N]\setminus\{i\}\\ |S|=k}}
(R_S)_{(j:k)}.
\end{equation}
The batch advantage is
\begin{equation}
\label{eq:method-aij-def}
a_{i,j}
\defeq
q_{i,j}-v^{(-i)}_j,
\qquad
a_i^{(\alpha)}
\defeq
\sum_{j=1}^k \alpha_j a_{i,j}.
\end{equation}
These subset averages are the unbarred analogues of the barred population quantities.  For an i.i.d. batch, they preserve unbiasedness by iterated expectation.  The leave-one-out baseline excludes sample $i$, so it is independent of the score term, matching the condition in \sect{sec:elementary-estimators}.

\subsection{Explicit batch equations and fast computation}
\label{sec:method-batch-equations}

This subsection gives efficient computation for the sample-based quantities defined above.  All batch quantities can be computed after a single sort.  Let
\[
R_{(1:N)}\le\cdots\le R_{(N:N)}
\]
be the sorted rewards, and map the resulting quantities back to the original datapoint indices by the inverse sorting permutation.

For the U-statistic value, the probability that sorted reward $R_{(m:N)}$ is the $j$th order statistic of a uniform size-$k$ subset is
\begin{equation}
\label{eq:method-W-batch}
W_{m,j}
=
\frac{\binom{m-1}{j-1}\binom{N-m}{k-j}}{\binom{N}{k}},
\qquad
m=1,\ldots,N,\quad j=1,\ldots,k.
\end{equation}
Therefore
\begin{equation}
\label{eq:method-v-sorted}
v_j=\sum_{m=1}^N R_{(m:N)}W_{m,j},
\qquad
v^\alpha=\sum_{m=1}^N R_{(m:N)}w_m^{(\alpha)},
\qquad
w^{(\alpha)}=W\alpha.
\end{equation}

For include-one values, let $t$ denote the sorted rank of the datapoint forced into the subset.  Define precomputable weight tables
\begin{equation}
\label{eq:method-ABC}
\begin{alignedat}{3}
A_{m,j}&\defeq
\frac{\binom{m-1}{j-1}\binom{N-m-1}{k-j-1}}{\binom{N-1}{k-1}},\quad&
B_{m,j}&\defeq
\frac{\binom{m-1}{j-1}\binom{N-m}{k-j}}{\binom{N-1}{k-1}},\quad&
C_{m,j}&\defeq
\frac{\binom{m-2}{j-2}\binom{N-m}{k-j}}{\binom{N-1}{k-1}}.
\end{alignedat}
\end{equation}
Then
\begin{equation}
\label{eq:method-q-sorted}
q_{t,j}
=
\sum_{m<t} R_{(m:N)}A_{m,j}
+
R_{(t:N)}B_{t,j}
+
\sum_{m>t} R_{(m:N)}C_{m,j}.
\end{equation}
For leave-one-out values, let
\begin{equation}
\label{eq:method-Wminus}
W^-_{\ell,j}
\defeq
\frac{\binom{\ell-1}{j-1}\binom{(N-1)-\ell}{k-j}}{\binom{N-1}{k}},
\qquad
\ell=1,\ldots,N-1.
\end{equation}
After removing sorted rank $t$,
\begin{equation}
\label{eq:method-vloo-sorted}
v^{(-t)}_j
=
\sum_{m<t} R_{(m:N)}W^-_{m,j}
+
\sum_{m>t} R_{(m:N)}W^-_{m-1,j}.
\end{equation}
Finally,
\begin{equation}
\label{eq:method-a-sorted}
a_{t,j}=q_{t,j}-v^{(-t)}_j,
\qquad
a_t^{(\alpha)}=\sum_{j=1}^k\alpha_j a_{t,j}.
\end{equation}
The appendix derives \eqref{eq:method-W-batch}--\eqref{eq:method-vloo-sorted} by counting subsets and gives prefix/suffix formulas.  Rankwise computation costs $O(Nk)$ after sorting.  For fixed $\alpha$, collapsed weights give all $a_t^{(\alpha)}$ in $O(N)$ after sorting, so the total minibatch cost is $O(N\log N)$.

\section{Small Diagnostic Experiments}
\label{sec:experiments}
\label{sec:experiments-small-diagnostics}

We first run small diagnostic experiments that isolate the two basic claims needed before the larger task evaluations: the batch quantities are cheap to compute, and the LR/RP estimators recover the intended order-statistic gradients.

\parag{Weight visualization and runtime.}
The first diagnostic studies the computation of the batch quantities from \sect{sec:method-batch-equations}.  The left panel of \fig{fig:diagnostic-runtime} visualizes the weights placed on sorted batch ranks when computing $v^\alpha$ for several choices of $\alpha$.  The middle panel measures the runtime of \texttt{lstat\_advantage} as the batch size $N$ grows.  The right panel fixes $N=5000$ and varies $k$, comparing the collapsed \texttt{lstat\_advantage} computation against the full rankwise order-statistic computation.  The results show that the extra OrderGrad computation is negligible: even for batch sizes around $10^4$, the update takes about $1$ ms, so we do not expect it to be a bottleneck in practice.

\begin{figure}[t]
  \centering
  \includegraphics[width=\linewidth]{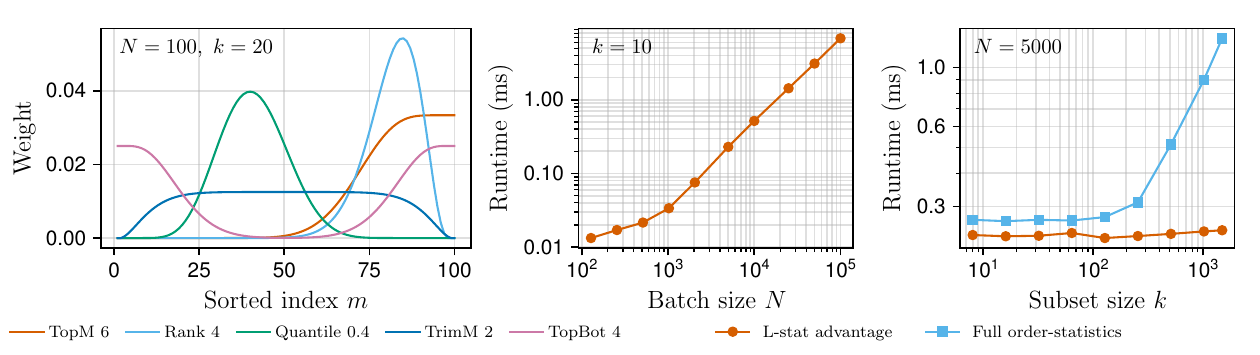} 
\caption{Diagnostic computation experiment.  The panels visualize several rank-weight choices for $v^\alpha$ and measure the runtime of the practical collapsed OrderGrad computation.}
\label{fig:diagnostic-runtime}
\end{figure}

\parag{Gradient bias, variance, and SNR.}
The second diagnostic checks gradient estimation in a one-dimensional setting where the target gradient can be computed accurately.  We use $x\sim\mathcal{N}(\mu,1)$ with reward $R(x)=-x^2$, compute the exact gradient of the $20\%$ CVaR objective with respect to $\mu$, and compare repeated LR and RP estimates across different values of $k$ and $N$.  The left panel of \fig{fig:diagnostic-gradient} reports the LR estimator variance as a function of $N$ for several choices of $k$.  The middle panel reports empirical bias against $k$, comparing the LR and RP estimates to the exact gradient.  Increasing $k$ increases variance, but it also reduces bias relative to the exact CVaR gradient.  Thus larger $k$ approximates the target objective more accurately, at the cost of higher estimation variance.

\parag{Top-$M$@$K$ SNR.}
The right panel of \fig{fig:diagnostic-gradient} studies the signal-to-noise ratio of the Top-$M$@$K$ estimator as $M$ varies.  This experiment isolates the effect of smoothing the sharp Max@$K$ objective: increasing $M$ averages over more high-ranked samples, which should reduce estimator noise while retaining pressure on the upper tail.

\begin{figure}[t]
\centering
\includegraphics[width=\linewidth]{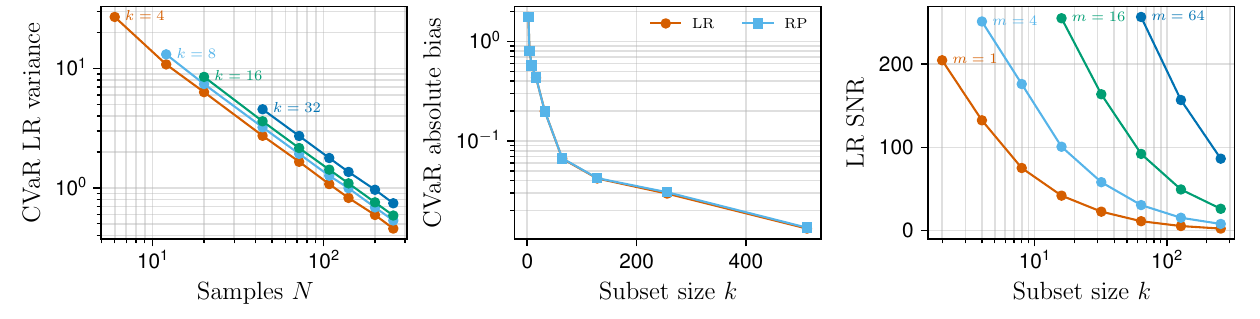}
\caption{Diagnostic gradient experiments.  For $x\sim\mathcal{N}(\mu=0.5,1)$ and $R(x)=-x^2$, we compare LR and RP estimates to the exact $20\%$ CVaR gradient and study the SNR of Top-$M$@$K$ estimators.}
\label{fig:diagnostic-gradient}
\end{figure}

\section{LLM Reasoning Experiments}
\label{sec:reasoning_experiments}

We evaluate our method on challenging math reasoning tasks. Our experiments
address three questions:
(i) does our method improve pass@$k$ on standard reasoning benchmarks compared
to GRPO and \maxpo{} (Max@$K$ Policy Optimization, a special case of OrderGrad; this can be viewed as \citet{chen2025pass}'s method without standard-deviation normalization);
(ii) how does the baseline parameter $m$ affect performance when the
objective size $K$ is fixed; and
(iii) how does the objective size $K$ affect performance when $m$ is fixed.

\subsection{Experimental Setting}
\label{sec:reasoning_setup}

\paragraph{Training.}
We perform RL fine-tuning on Qwen3-4B-Base~\citep{yang2025qwen3} and
Qwen2.5-Math-7B~\citep{yang2025qwen25math}.
We compare our method against two baselines:
GRPO~\citep{shao2024deepseekmath} and \maxpo{}~\citep{chen2025pass}.
\maxpo{} corresponds to the special case of our top-$m$ method with $m=1$,
since both reduce to subtracting the maximum reward of the leave-one-out group.
Unless otherwise stated, the main experiments use $K=4$ and our method with
Top$2$@$4$, with a full sensitivity analysis on $m$ and $K$ reported in
Sec.~\ref{sec:ablation_hyperparam}. Training data and hyperparameters are
provided in App.~\ref{app:llm:hparams}.
Throughout this section, $K$ refers to the training-time Max@$K$ objective size
and $m$ to the number of top-ranked rewards averaged in our baseline.

\paragraph{Evaluation.}
We evaluate on five math reasoning benchmarks: AIME24, AIME25, AMC23,
MATH500~\citep{hendrycks2021measuring}, and Minerva~\citep{lewkowycz2022solving}.
All evaluations use nucleus sampling with temperature $0.6$ and top-$p$ $0.95$.
To reduce evaluation variance, we generate $n=1024$ samples per problem.
We report the unbiased pass@$k$~\citep{chen2021evaluating} metric for
$k\in\{1,2,4,8,\ldots\}$, computed as
\begin{equation}
\text{pass@}k \;:=\; \mathbb{E}_{x\sim\mathcal{D}}\!\left[\,1-\frac{\binom{n-c}{k}}{\binom{n}{k}}\,\right],
\label{eq:passk}
\end{equation}
where $n$ is the number of sampled completions and $c$ is the number of correct
completions among them. We report results up to $k\le 256$ for all benchmarks.

\subsection{Results}
\label{sec:reasoning_main_results}

Figure~\ref{fig:passk_main} summarizes the task-average pass@$k$ curves.  Against GRPO, OrderGrad gives higher large-$k$ performance: pass@$256$ improves by $0.092$ on Qwen3-4B-Base and $0.022$ on Qwen2.5-Math-7B, and it avoids GRPO's diversity collapse on Qwen3-4B-Base.  Against \maxpo{}~($K=4$), Top$2$@$4$ improves most of the pass@$k$ profile but not every endpoint: pass@$1$ improves by $0.057$ on Qwen3-4B-Base and $0.018$ on Qwen2.5-Math-7B, and pass@$256$ improves by $0.026$ on Qwen3-4B-Base, while the largest-$k$ Qwen2.5-Math-7B results are comparable to \maxpo{}~($K=4$) rather than uniformly higher.  Per-benchmark curves and pass@$1$/pass@$256$ are provided in App.~\ref{app:llm:additional_results}
(Figures~\ref{fig:per_task_qwen}--\ref{fig:per_task_qwen3} and
Table~\ref{tab:perbench_pass1_pass256}).

\begin{figure}[t]
  \centering
  \begin{subfigure}[t]{0.48\linewidth}
    \centering
    \includegraphics[width=\linewidth]{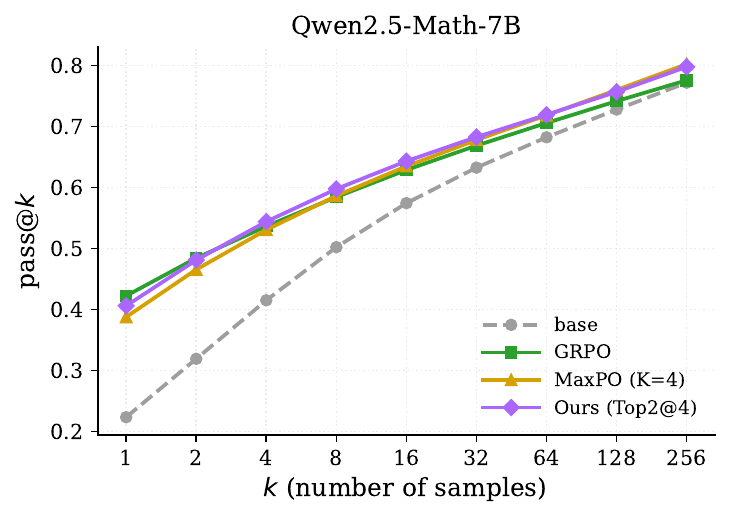}
    \caption{Qwen2.5-Math-7B}
    \label{fig:passk_main_qwen}
  \end{subfigure}\hfill
  \begin{subfigure}[t]{0.48\linewidth}
    \centering
    \includegraphics[width=\linewidth]{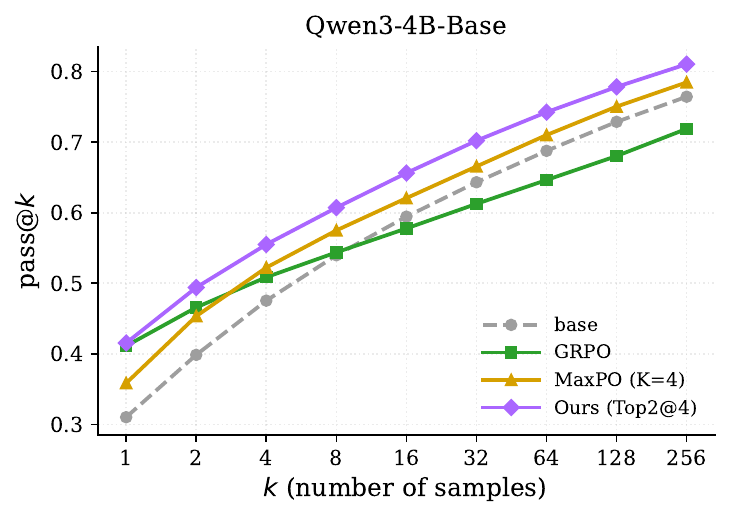}
    \caption{Qwen3-4B-Base}
    \label{fig:passk_main_qwen3}
  \end{subfigure}
  \caption{Task-average pass@$k$ ($k\le 256$). Unweighted average over AIME24,
  AIME25, AMC23, MATH500, and Minerva (temperature $0.6$, top-$p$ $0.95$,
  $n=1024$ per problem). Our method with Top$2$@$4$ outperforms GRPO at
  large $k$ and outperforms \maxpo{}~($K=4$) overall on pass@$k$.}
  \label{fig:passk_main}
\end{figure}

\subsection{Effective Size of $m$, $K$}
\label{sec:ablation_hyperparam}

We analyze the sensitivity of our method to the two hyperparameters on Qwen3-4B-Base
(Figure~\ref{fig:hparam_ablation}): the baseline parameter $m$ at fixed
$K=6$, and the training objective size $K$ at fixed Top $m=2$. Corresponding
results on Qwen2.5-Math-7B are reported in
App.~\ref{app:llm:additional_results}
(Tables~\ref{tab:taskavg_qwen} and~\ref{tab:perbench_pass1_pass256}).

\paragraph{Effect of $m$ ($K=6$ fixed).}
We compare $m\in\{2,3\}$ against \maxpo{}~($K=6$) (Figure~\ref{fig:effN_K6}).
Both configurations improve pass@$1$ by more than $0.05$ and pass@$256$ by
$0.034$--$0.043$ over \maxpo{}~($K=6$) on Qwen3-4B-Base.
Our baseline averages the top-$m$ rewards of the leave-one-out group; larger $m$ averages over more samples and yields a better baseline estimate.

\paragraph{Effect of $K$ ($m=2$ fixed).}
We compare $K\in\{4,6\}$ (Figure~\ref{fig:effK_N2}). Increasing $K$ from $4$
to $6$ decreases pass@$1$ while increasing pass@$256$.
This indicates an exploration--exploitation trade-off: increasing $K$ promotes
broader exploration, which reduces one-shot accuracy but improves the
probability of success when multiple samples are available.

\begin{figure}[t]
  \centering
  \begin{subfigure}[t]{0.48\linewidth}
    \centering
    \includegraphics[width=\linewidth]{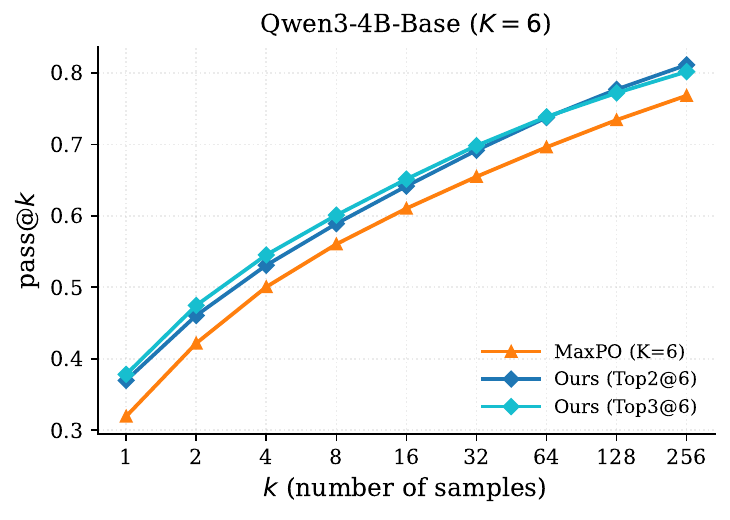}
    \caption{Effect of $m$ at fixed $K=6$.}
    \label{fig:effN_K6}
  \end{subfigure}\hfill
  \begin{subfigure}[t]{0.48\linewidth}
    \centering
    \includegraphics[width=\linewidth]{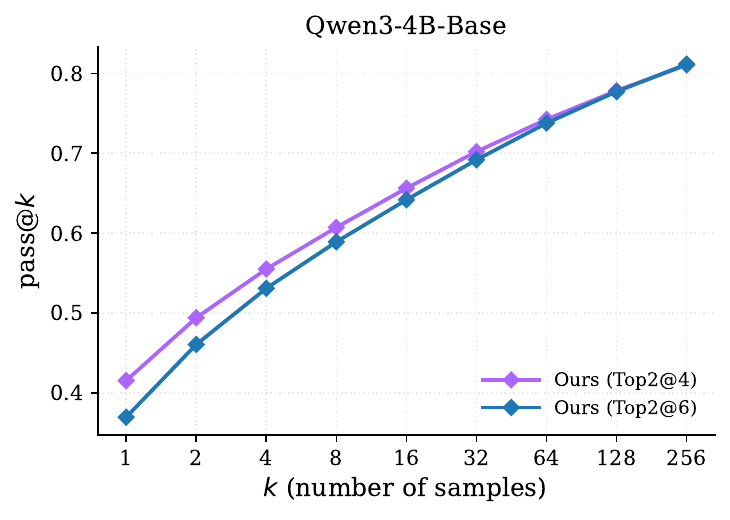}
    \caption{Effect of $K$ at fixed $m=2$.}
    \label{fig:effK_N2}
  \end{subfigure}
  \caption{Effective size of $m$ and $K$ on Qwen3-4B-Base.}
  \label{fig:hparam_ablation}
\end{figure}

\subsection{Different Reward Objectives}
\begin{figure}[t]
  \centering
    \begin{minipage}{\linewidth}
      \centering
      \begin{subfigure}[t]{0.41\linewidth}
        \centering
        \includegraphics[width=\linewidth]{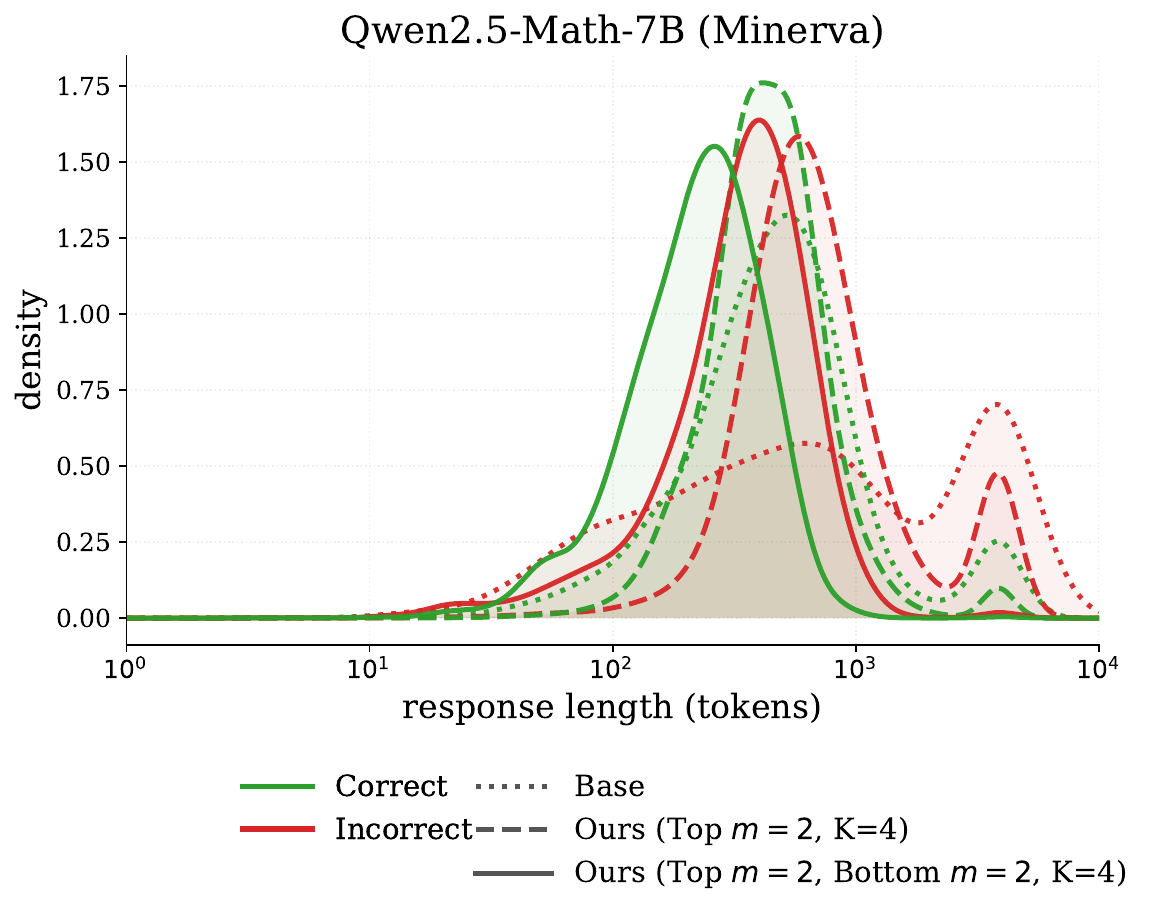}
        \caption{Response length distribution by correct and incorrect subsets on Minerva.}
        \label{fig:length_dist}
      \end{subfigure}\hfill
      \begin{subfigure}[t]{0.55\linewidth}
        \centering
        \includegraphics[width=\linewidth]{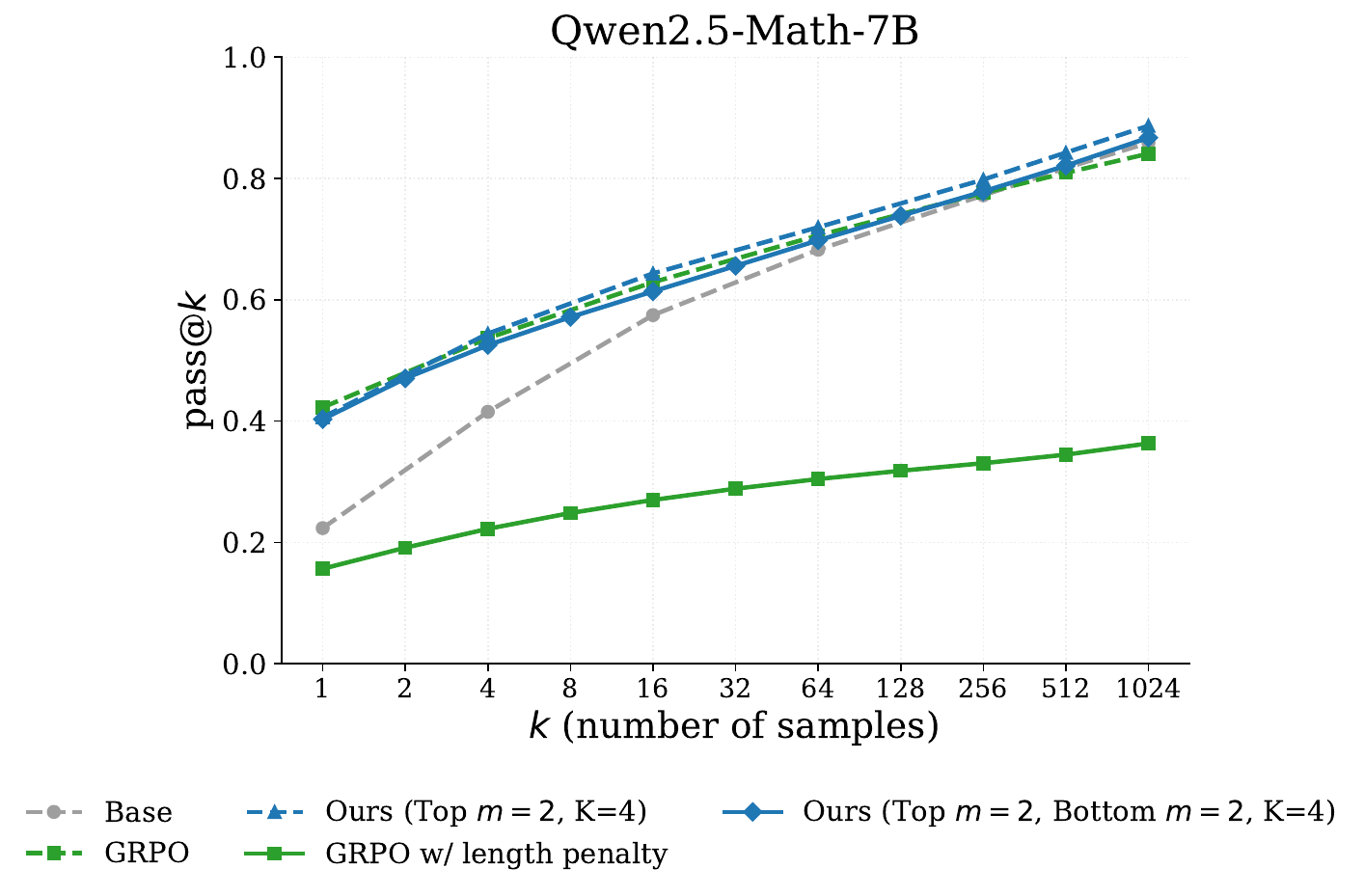}
        \caption{Task-average pass@$k$ ($k\le 1024$). Unweighted average over AIME24,
        AIME25, AMC23, MATH500, and Minerva}
        \label{fig:passk_length_avg}
      \end{subfigure}
    \end{minipage}  \caption{Results for Multi-Reward Objectives with Correctness Reward and Length Penalty (temperature $0.6$, top-$p$ $0.95$,
  $n=1024$ per problem).}
  \label{fig:topbottom}
\end{figure}

We conduct experiments using a correctness reward for the Top $m$ rollouts and a response-length reward for the Bottom $m$ rollouts (shorter responses receive higher reward). In RL for LLMs, such as GRPO \citep{shao2024deepseekmath} and DAPO \citep{yu2025dapo}, a common approach is to apply a weighted sum scalarization of correctness and length rewards uniformly to all rollouts. However, interference from the length penalty can lead to degenerate short outputs. Moreover, LLMs tend to overthink \citep{sui2025stop}
when producing incorrect responses (Fig.~\ref{fig:length_dist}), and the longest rollout can become the bottleneck during Best-of-$N$ inference. We compare two objectives. The GRPO baseline uses $r = r_{\mathrm{correct}} - 0.2 \cdot L/L_{\max}$, where correctness reward $r_{\mathrm{correct}} \in \{0,1\}$, $L$ denotes the response length, and $L_{\max}=3072$. OrderGrad uses $K=4$, Top $m=2$, and Bottom $m=2$. The top two
rollouts use a correctness weight of $1.0$, while the two longest rollouts use a length-penalty weight of $0.4 \cdot L/L_{\max}$. As shown in Fig.~\ref{fig:passk_length_avg}, GRPO is adversely affected by the length penalty, leading to poor performance and output collapse. In contrast, OrderGrad preserves pass@$k$
performance close to the $K=4$, $m=2$ setting while greatly reducing the response length; in particular, extremely long responses are eliminated (Fig.~\ref{fig:length_dist}).

\section{Discussion, Limitations and Conclusions}
\label{sec:discussion}

OrderGrad exposes a large design space: choosing $k$ and $\alpha$ determines where gradient pressure falls across the reward distribution.  Smaller $k$ and smoother weights can stabilize learning, while larger $k$ and sharper weights better match extreme deployment metrics but may increase variance.  We explored only a focused subset over $k$, $\alpha$, reward definitions, and task domains.  Practical use also requires careful consistent tie/atom handling, and accounting for off-policy data, reward-model misspecification, prompt dependence, and PPO clipping/normalization in LLM settings.

OrderGrad is a powerful method that unifies robust, safe, and exploratory metrics under one rank-weighted gradient-estimation framework.  The method is flexible and easy to use because it can be implemented as a reward or advantage transformation and plugged into standard LR or RP updates.  We hope the community finds OrderGrad useful for optimizing beyond the mean.

\section*{Author Contributions}

{\bf Paavo Parmas}: Conceptualized the method, all theoretical derivations, created the OrderGrad code, toy experiments, proposed the TopM-BotM joint experiment, lead the project, lead writer.

Yongmin Kim: Conducted the overall LLM experiments, engineering, writing regarding the LLM experimental sections.

Kohsei Matsutani: Conducted the TopM-BottomM experiments and wrote the corresponding sections, wrote parts of the related work.

Shota Takashiro: Ported the OrderGrad code to the LLM setting, discussed the LLM experiments.

Soichiro Nishimori: Ran the RL experiments in the appendix, participated in discussions.

Takeshi Kojima: Overall student supervision regarding LLM topics in the lab, and discussions regarding the LLM sections.

Yusuke Iwasawa: funding acquisition, overall student supervision and management in the lab.

Yutaka Matsuo: funding acquisition, overall student supervision and management in the lab.

\begin{ack}
Paavo Parmas was supported by JST ACT-X, Japan, Grant Number JPMJAX23CO.
\end{ack}

\bibliographystyle{apalike}
\bibliography{refs}

\clearpage
\appendix
\renewcommand{\appendixpagename}{\Large Appendix}
\appendixpage
\addappheadtotoc
\startcontents[appendices]
\printcontents[appendices]{l}{1}{\setcounter{tocdepth}{2}}
\clearpage
\section{Extended Related Work}
\label{sec:related-work}

\paragraph{Learning criteria beyond the mean.}
This work is motivated by the view that learning should target desirable properties of the loss or reward distribution, not only the mean \citep{holland_tanabe_2023_survey}.  Order-based criteria sort outcomes and weight them by rank; these include L-statistics and L-estimates from robust statistics \citep{daniell_1920_order,bickel_lehmann_1975_location,huber_ronchetti_2009_robust}, robust L-statistic learning \citep{maurer_2020_lstat}, and rank-weighted learning theory \citep{khim_2020_rank_weighted,liu_2022_general_risk}.

\paragraph{Risk measures, quantiles, CVaR, and spectral risks.}
CVaR/expected shortfall and spectral risk measures can be written as weighted integrals of quantiles \citep{acerbi_2002_spectral,acerbi_tasche_2002_expected,acerbi_tasche_2002_coherence,rockafellar_uryasev_2000_cvar,rockafellar_uryasev_2002_cvar_general,shapiro_2013_kusuoka}.  Risk-sensitive and CVaR optimization have been studied in stochastic optimization and RL \citep{chow2015risk,chow2018risk,cardoso_xu_2019_risk_bandit,curi_2020_risk_averse,holland_haress_2021_risk_averse,holland_haress_2022_spectral}.  OrderGrad differs in focusing on a finite-$k$ gradient estimator that can target arbitrary L-statistic weights.

\paragraph{Distributional reinforcement learning.}
Distributional reinforcement learning treats the discounted return as a random variable and aims to model its probability law, rather than collapsing it immediately to an expected value.  Earlier work used such distributional information for robustness and risk-aware decision making before the now-standard deep-RL formulation of the area emerged.  For instance, \citet{sugiyama_2010_lapi} introduced least absolute policy iteration as a more robust counterpart to least-squares value-function approximation (though note that their
appendix also has methods for quantile estimation, etc.), while \citet{morimura_2010_parametric_return_density} and \citet{morimura_2010_nonparametric_return_distribution} estimated return densities and return distributions directly, illustrating how distributional estimates can reveal features of performance that are invisible from the mean alone.  This perspective is naturally connected to risk-sensitive RL, but the two are not identical in emphasis: risk-sensitive methods usually specify a scalar criterion in advance, such as a variance penalty, a percentile, or CVaR, whereas distributional RL first represents a richer object---the law of the return---from which many such criteria can subsequently be extracted.

The contemporary deep distributional-RL literature was largely shaped by the distributional Bellman formulation and the C51 algorithm of \citet{bellemare_2017_distributional}.  That work showed that approximating value distributions can meaningfully affect empirical learning behavior, even when the final control objective is still expected return.  Later analyses clarified why the categorical projection used in C51 is well behaved and related it to the Cram\'er distance \citep{rowland_2018_categorical_distributional}.  Another major family of methods represents the quantile function, or inverse CDF, rather than a categorical distribution.  QR-DQN estimates return values at a fixed set of quantile levels using quantile regression \citep{dabney_2018_quantile_regression_drl}; IQN conditions on sampled quantile fractions to obtain an implicit quantile-function representation \citep{dabney_2018_iqn}; and FQF further learns the quantile fractions themselves \citep{yang_2019_fqf}.  These quantile-based approaches are particularly relevant to OrderGrad because L-statistic and spectral objectives can be expressed as weighted integrals over quantiles.

Distributional RL has since been developed in several additional directions.  D4PG incorporated distributional critics into off-policy actor-critic learning for continuous control \citep{barth_maron_2018_d4pg}, and Rainbow showed that a categorical distributional update can be combined effectively with other improvements to deep Q-learning \citep{hessel_2018_rainbow}.  A theoretical line of work investigates which summaries of the return distribution are propagated by distributional dynamic programming, including statistics-and-samples formulations and moment-matching approaches \citep{rowland_2019_statistics_samples,nguyen_tang_2021_moment}.  Distributional value estimates have also been used to encourage exploration by exploiting upper-tail behavior or variability in the learned return distribution \citep{mavrin_2019_distributional_exploration}.  At the same time, comparative studies indicate that the benefits of distributional learning are sensitive to the approximation architecture and learning regime, rather than being an automatic consequence of representing more information \citep{lyle_2019_comparative}.  For a comprehensive account of the mathematical and algorithmic foundations of the field, see \citet{bellemare_2023_distributional_book}.

OrderGrad is best viewed as complementary to this body of work.  In typical distributional-RL methods, a critic estimates a return distribution, or selected statistics of that distribution, and the resulting policy may still be chosen by maximizing expected return.  OrderGrad instead provides a direct policy-gradient estimator for a specified finite-$k$ order-statistic objective.  A distributional critic could therefore be useful for choosing, estimating, or interpreting the reward distribution associated with a target objective, but OrderGrad itself does not require a parametric model of the full return law.  Conversely, OrderGrad offers a policy-gradient mechanism for directly optimizing deployment-oriented criteria such as Top-$M$@$K$, best-of-$K$, lower-tail CVaR-like averages, medians, trimmed means, and other rank-weighted objectives that distributional RL is often able to represent but does not necessarily optimize directly.

\paragraph{Top-$k$, maximum losses, and best-of-$k$ training.}
Average top-$k$ and maximum-loss objectives appear in supervised learning and robust optimization \citep{fan_2017_topk,shalev_wexler_2016_maximal,ogryczak_tamir_2003_sum_k_largest}.  OrderGrad treats Max@$k$ and Top-$M$@$k$ as instances of the same order-statistic gradient framework and emphasizes their use as reward transformations in policy-gradient and generative-model training.

\paragraph{Gradient estimators.}
Score-function, infinitesimal perturbation, and pathwise derivative estimators are classical tools for stochastic simulation and learning \citep{lecuyer_1990_unified,williams_1992_reinforce,mohamed_2020_mc_gradient}.  Reparameterization gradients are standard in latent-variable generative models \citep{kingma_welling_2014_autoencoding,rezende_2014_stochastic}, and recent work studies unified or composite estimators that combine LR and RP signals \citep{parmas_2021_unified,parmas_seno_2022_proppo}.  OrderGrad is complementary: it contributes the conditional order-statistic advantage needed to apply these estimator families to L-statistic objectives.

\paragraph{RLVR in LLMs}
Reinforcement learning with verifiable rewards (RLVR) \citep{lambert2025tulu,guo2025deepseek} is a training paradigm that optimizes LLM policies using deterministic feedback from objective verifiers rather than subjective human preference \citep{paul2017deep}. Despite this promise \citep{guo2025deepseek,jaech2024openai}, recent studies suggest that RLVR in LLMs primarily serves to amplify reasoning behaviors already present in the base model \citep{liu2025understanding,zhao2025echochamber,ai2025rethinking}. \citet{yue2025does} investigated the pass@$k$ metric \citep{chen2021evaluating,song2025mind,dang2025weight,wen2025reinforcement,wu2025invisible} and found experimentally that as $k$ increases, the base model's pass@$k$ eventually exceeds that of the model after RLVR. This phenomenon has been linked to diversity collapse \citep{dang2025weight,cui2025entropy} and squeezing reasoning paths \citep{matsutani2025rl,hu2025how,bu2025consistency}. \citet{wang2025octothinker,zhang2025learning,chen2025the,zhang2025interplay} argue that mid-training is a prerequisite for effective RLVR, as it underperforms on Llama \citep{grattafiori2024llama3} compared to Qwen \citep{yang2025qwen25,yang2025qwen3}.

\paragraph{Exploration in RL for LLMs}
Given these limitations, prior studies have incorporated exploration into RLVR. \citet{cui2025entropy,cheng2025reasoning,zheng2025first,shen2025entropy,jiang2025rethinking} leveraged entropy bonuses to encourage exploration via policy uncertainty. \citet{song2025outcome} proposed outcome-based exploration using UCB-style bonuses \citep{auer2002finite}. \citet{yu2025restrain} applied self-penalization by assigning negative rewards to high-confidence answers that deviate from the majority consensus. \citet{he2025rewarding} improved exploration by up-weighting low probability but correct trajectories, and \citet{gao2025navigate} adopted Random Network Distillation (RND) \citep{burda2018exploration} to provide bonuses for unknown trajectories. \citet{zhou2025evolving} augmented training with a semantic novelty score computed from embeddings, \citet{li2025jointly} employed a semantic diversity score with an external semantic comparator, and \citet{tuyls2025representation} computed a representation-based novelty score from hidden states to boost exploration. \citet{liang2025can} leveraged reward model gradients to improve temperature sampling. \citet{setlur2025e3} promoted in-context exploration via skill asymmetries and negative gradients.

\section{Notation}
\label{app:notation-ranknote}

This appendix gives detailed derivations behind the main text.  We follow the organization: first an exact known-$(\bar R,p)$ regime, then a sampling regime with an i.i.d. batch.  The order-statistic sample size is $k$ and the optimization batch size is $N$.

\paragraph{Regime A: known finite distribution.}
There are $m$ arms.  Arm $b\in\{1,\ldots,m\}$ has deterministic reward $\bar R_b\in\R$ and sampling probability $p_b$, with $p\in\Delta^{m-1}$.  We draw
\[
A_1,\ldots,A_k\stackrel{i.i.d.}{\sim}p,
\qquad
R_\ell\defeq \bar R_{A_\ell},
\qquad
R_{(1:k)}\le\cdots\le R_{(k:k)}.
\]
For each rank $j\in\{1,\ldots,k\}$, the exact population value, conditional value, and arm advantage are
\[
\bar v_j\defeq \E[R_{(j:k)}],
\qquad
\bar q_{b,j}\defeq \E[R_{(j:k)}\mid A_1=b],
\qquad
\bar a_{b,j}\defeq \bar q_{b,j}-\bar v_j.
\]
For L-statistic weights $\alpha\in\R^k$,
\[
\bar v^\alpha\defeq \sum_{j=1}^k\alpha_j\bar v_j,
\qquad
\bar q_b^{(\alpha)}\defeq\sum_{j=1}^k\alpha_j\bar q_{b,j},
\qquad
\bar a_b^{(\alpha)}\defeq\sum_{j=1}^k\alpha_j\bar a_{b,j}.
\]

\paragraph{Regime B: sampled batch.}
In the sampling regime we observe
\[
A^{(1)},\ldots,A^{(N)}\stackrel{i.i.d.}{\sim}p,
\qquad
R_i\defeq \bar R_{A^{(i)}}.
\]
Batch-computable quantities are unbarred.  The value estimator $v_j$ is defined for $1\le k\le N$; the LR leave-one-out advantage uses $1\le k<N$.  Within this range, $v_j$ is the U-statistic value estimator, $q_{i,j}$ is the include-one subset expectation, $v^{(-i)}_j$ is the leave-one-out subset expectation, and
\[
a_{i,j}\defeq q_{i,j}-v^{(-i)}_j,
\qquad
v^\alpha\defeq\sum_{j=1}^k\alpha_j v_j,
\qquad
a_i^{(\alpha)}\defeq\sum_{j=1}^k\alpha_j a_{i,j}.
\]

\paragraph{Ties.}
Equal rewards are not collapsed.  We use a stable total ordering of arms or batch positions.  With ties, intermediate CDF tables should be interpreted as stable-rank CDFs rather than numeric CDFs; because tied entries have identical rewards, the final expectations remain unchanged.

\subsection*{Binomial-tail convention}
Throughout, $\operatorname{Bin}(n,\pi)$ denotes a binomial random variable, and
\[
\Prob(\operatorname{Bin}(n,\pi)\ge s)=1\quad(s\le 0),
\qquad
\Prob(\operatorname{Bin}(n,\pi)\ge s)=0\quad(s>n).
\]

\section{Exact known-$(\bar R,p)$ formulas}
\label{app:known-distribution-details}

Let $\pi$ be a stable sorting permutation of arms:
\[
\bar R_{(1)}\le\bar R_{(2)}\le\cdots\le\bar R_{(m)},
\qquad
p_{(t)}\defeq p_{\pi(t)}.
\]
Let $\rho(b)$ be the stable sorted position of arm $b$, so that $\pi(\rho(b))=b$.  Define the stable-rank CDF grid
\[
F_0\defeq0,
\qquad
F_t\defeq \sum_{s=1}^t p_{(s)},
\qquad t=1,\ldots,m.
\]
When all rewards are distinct, $F_t=\Prob(R\le \bar R_{(t)})$.  With ties, $F_t$ is the probability that the sampled arm lies among the first $t$ entries of the stable sorted list.

\subsection{Unconditional order statistics}
\label{app:known-unconditional}

Fix $j\in\{1,\ldots,k\}$.  For each stable threshold $t$, define the count
\[
Y_t\defeq \#\{\ell\in\{1,\ldots,k\}: \rho(A_\ell)\le t\}.
\]
Then $Y_t\sim\operatorname{Bin}(k,F_t)$.

\begin{lemma}[Order-statistic CDF as a binomial tail]
\label{lem:app-bin-tail-uncond}
For each $t=0,\ldots,m$,
\[
\Prob\{\text{the stable rank of }R_{(j:k)}\text{ is at most }t\}
=
\Prob\{\operatorname{Bin}(k,F_t)\ge j\}.
\]
When rewards are distinct, this event is exactly $\{R_{(j:k)}\le \bar R_{(t)}\}$.
\end{lemma}

\begin{proof}
The $j$th order statistic has stable rank at most $t$ if and only if at least $j$ of the $k$ sampled arms have stable rank at most $t$.  That count is $Y_t$, which is binomial with parameters $(k,F_t)$.
\end{proof}

Define the tail table
\begin{equation}
\label{eq:app-Tk}
T^{(k)}_{t,j}
\defeq
\Prob\{\operatorname{Bin}(k,F_t)\ge j\},
\qquad
t=0,\ldots,m,\quad j=1,\ldots,k.
\end{equation}
Differencing adjacent CDF values gives the probability that stable sorted entry $t$ supplies the $j$th order statistic:
\begin{equation}
\label{eq:app-Wwr}
W^{\mathrm{wr}}_{t,j}
\defeq
T^{(k)}_{t,j}-T^{(k)}_{t-1,j},
\qquad
t=1,\ldots,m.
\end{equation}
Therefore
\begin{equation}
\label{eq:app-vbar}
\bar v_j
=
\sum_{t=1}^m \bar R_{(t)}W^{\mathrm{wr}}_{t,j}.
\end{equation}

\paragraph{Matrix form.}
Let $\bar R_{(\cdot)}=(\bar R_{(1)},\ldots,\bar R_{(m)})^\top$.  Let $T^{(k)}\in\R^{(m+1)\times k}$ collect \eqref{eq:app-Tk}, and let $D\in\R^{m\times(m+1)}$ be the forward-difference operator
\[
D_{t,t}=1,
\qquad
D_{t,t-1}=-1,
\qquad
\text{all other entries }0.
\]
Then $W^{\mathrm{wr}}=DT^{(k)}$ and
\begin{equation}
\label{eq:app-vbar-matrix}
\bar v=(\bar v_1,\ldots,\bar v_k)
=
\bar R_{(\cdot)}^\top W^{\mathrm{wr}}
=
\bar R_{(\cdot)}^\top DT^{(k)}.
\end{equation}

\subsection{Conditional-on-first-draw order statistics}
\label{app:known-conditional}

Fix arm $b$.  Define the stable-rank indicator
\[
\delta_{b,t}\defeq \ind\{\rho(b)\le t\}.
\]
Conditioning on $A_1=b$ pins one draw and leaves $A_2,\ldots,A_k$ i.i.d.  For a stable threshold $t$, the pinned draw contributes $\delta_{b,t}$ to the count below the threshold, while the remaining $k-1$ draws contribute a binomial count with success probability $F_t$.  Hence
\begin{equation}
\label{eq:app-bin-tail-cond}
\Prob\{\text{stable rank of }R_{(j:k)}\text{ is at most }t\mid A_1=b\}
=
\Prob\{\operatorname{Bin}(k-1,F_t)\ge j-\delta_{b,t}\}.
\end{equation}
Define
\begin{equation}
\label{eq:app-Tb}
T^{(b)}_{t,j}
\defeq
\Prob\{\operatorname{Bin}(k-1,F_t)\ge j-\delta_{b,t}\}.
\end{equation}
The conditional mass table is
\begin{equation}
\label{eq:app-Wb}
W^{(b)}_{t,j}
\defeq
T^{(b)}_{t,j}-T^{(b)}_{t-1,j},
\end{equation}
and therefore
\begin{equation}
\label{eq:app-qbar}
\bar q_{b,j}
=
\sum_{t=1}^m \bar R_{(t)}W^{(b)}_{t,j}.
\end{equation}
In matrix form,
\begin{equation}
\label{eq:app-qbar-matrix}
\bar q_b=(\bar q_{b,1},\ldots,\bar q_{b,k})
=
\bar R_{(\cdot)}^\top W^{(b)}
=
\bar R_{(\cdot)}^\top DT^{(b)}.
\end{equation}

\paragraph{Shared precomputation.}
Let
\[
G_{t,s}\defeq \Prob\{\operatorname{Bin}(k-1,F_t)\ge s\}.
\]
Then $G$ can be precomputed once for all $t$ and integer thresholds $s$, and
\[
T^{(b)}_{t,j}=G_{t,j-\delta_{b,t}}.
\]
Conditioning on arm $b$ only shifts the rank threshold by $0$ or $1$ at each stable sorted grid point.

\subsection{Mixture identity, advantages, and L-statistics}
\label{app:known-mixture}

\begin{proposition}[Mixture identity]
\label{prop:app-mixture}
For every rank $j$,
\[
\bar v_j=\sum_{b=1}^m p_b\bar q_{b,j}.
\]
\end{proposition}

\begin{proof}
By the law of total expectation over the first sampled arm,
\[
\E[R_{(j:k)}]
=
\sum_{b=1}^m
\E[R_{(j:k)}\mid A_1=b]\Prob(A_1=b)
=
\sum_{b=1}^m p_b\bar q_{b,j}.
\]
\end{proof}

The exact arm advantage is
\begin{equation}
\label{eq:app-adv-known}
\bar a_{b,j}\defeq \bar q_{b,j}-\bar v_j,
\qquad
\bar a_b=(\bar a_{b,1},\ldots,\bar a_{b,k}).
\end{equation}
Consequently $\sum_{b=1}^m p_b\bar a_{b,j}=0$ for every $j$.  For an L-statistic,
\begin{equation}
\label{eq:app-known-lstat}
\bar v^\alpha=\alpha^\top\bar v,
\qquad
\bar q_b^{(\alpha)}=\alpha^\top\bar q_b,
\qquad
\bar a_b^{(\alpha)}=\alpha^\top\bar a_b.
\end{equation}

\subsection{Known-regime likelihood-ratio gradient}
\label{app:known-lr-gradient}

Assume the rewards $\bar R_b$ are fixed and the probabilities $p_\theta(b)$ are differentiable with common support.  Let
\[
T_\alpha(A_{1:k})
\defeq
\sum_{j=1}^k \alpha_j R_{(j:k)}.
\]
Then $\bar v^\alpha(\theta)=\E_\theta[T_\alpha(A_{1:k})]$.  The score-function identity gives
\begin{align}
\nabla_\theta\bar v^\alpha(\theta)
&=
\E_\theta\!\left[
T_\alpha(A_{1:k})
\sum_{\ell=1}^k
\nabla_\theta\log p_\theta(A_\ell)
\right]\\
&=
k\,\E_\theta\!\left[
T_\alpha(A_{1:k})
\nabla_\theta\log p_\theta(A_1)
\right]\\
&=
k\sum_{b=1}^m
p_\theta(b)\bar q_b^{(\alpha)}
\nabla_\theta\log p_\theta(b).
\end{align}
Because $\sum_b p_\theta(b)\nabla_\theta\log p_\theta(b)=0$, subtracting the baseline $\bar v^\alpha$ gives
\begin{equation}
\label{eq:app-known-lr-gradient}
\nabla_\theta\bar v^\alpha(\theta)
=
k\sum_{b=1}^m
p_\theta(b)\bar a_b^{(\alpha)}
\nabla_\theta\log p_\theta(b).
\end{equation}

\section{Batch values and advantage computation}
\label{app:advantage-computation}

Throughout this section, the realized batch rewards $R_1,\ldots,R_N$ are fixed and subset expectations are with respect to uniform subset selection without replacement.  Write $[N]\defeq\{1,\ldots,N\}$.

\subsection{U-statistic values}
\label{app:batch-ustat}

For $1\le k\le N$, define
\begin{equation}
\label{eq:app-Ustat-def}
v_j
\defeq
\frac{1}{\binom{N}{k}}
\sum_{\substack{S\subseteq [N]\\|S|=k}}
(R_S)_{(j:k)}.
\end{equation}
Equivalently, if $S$ is a uniform size-$k$ subset of $[N]$ independent of the rewards, then
\[
v_j=\E[(R_S)_{(j:k)}\mid R_{1:N}].
\]

\begin{lemma}[Random subsample of i.i.d. observations]
\label{lem:app-subsample-iid}
Let $Z_1,\ldots,Z_M$ be i.i.d. and let $J$ be a uniformly random subset of $\{1,\ldots,M\}$ of size $s$, chosen independently of $Z_{1:M}$.  For any symmetric function $h$,
\[
\E[h(Z_J)]=\E[h(Z_1,\ldots,Z_s)].
\]
\end{lemma}

\begin{proof}
Condition on the subset $J$.  Since $Z_{1:M}$ are i.i.d., the selected values have the same joint distribution as the first $s$ values, up to ordering.  Averaging over $J$ preserves this distributional equality for symmetric functions.
\end{proof}

\begin{theorem}[Unbiasedness of the U-statistic value estimator]
\label{thm:app-Ustat-unbiased}
If $R_1,\ldots,R_N$ are i.i.d. with the same distribution as $R$, then
\[
\E[v_j]=\bar v_j.
\]
\end{theorem}

\begin{proof}
Let $S$ be a uniform size-$k$ subset independent of the rewards.  Then
\[
\E[v_j]=\E[(R_S)_{(j:k)}].
\]
By Lemma~\ref{lem:app-subsample-iid}, the multiset $\{R_i:i\in S\}$ has the same distribution as $k$ fresh i.i.d. rewards.  Therefore $(R_S)_{(j:k)}$ has the same distribution as $R_{(j:k)}$, and the expectation is $\bar v_j$.
\end{proof}

\subsection{Sort-plus-weights formula for values}
\label{app:batch-sort-weights}

After sorting the batch as $R_{(1:N)}\le\cdots\le R_{(N:N)}$, the probability that position $m$ supplies the $j$th order statistic of a uniform size-$k$ subset is
\begin{equation}
\label{eq:app-W-batch}
W_{m,j}
=
\frac{\binom{m-1}{j-1}\binom{N-m}{k-j}}{\binom{N}{k}}.
\end{equation}
Indeed, the subset must include position $m$, choose $j-1$ positions from the $m-1$ lower positions, and choose $k-j$ positions from the $N-m$ higher positions.  Consequently
\begin{equation}
\label{eq:app-v-batch-matrix}
v_j=\sum_{m=1}^N R_{(m:N)}W_{m,j},
\qquad
v=(v_1,\ldots,v_k)=\mathbf{R}_{(\cdot)}^\top W,
\end{equation}
where $\mathbf{R}_{(\cdot)}=(R_{(1:N)},\ldots,R_{(N:N)})^\top$.  For L-statistic weights,
\begin{equation}
\label{eq:app-valpha-batch}
v^\alpha
=
\alpha^\top v
=
\mathbf{R}_{(\cdot)}^\top W\alpha.
\end{equation}

\subsection{Include-one and leave-one-out definitions}
\label{app:batch-include-loo-defs}

Assume $1\le k<N$.  Let $S$ be a uniform size-$k$ subset of $[N]$.  For each original index $i$,
\begin{align}
\label{eq:app-def-include-loo}
q_{i,j}
&\defeq
\E[(R_S)_{(j:k)}\mid i\in S,\ R_{1:N}],\\
v^{(-i)}_j
&\defeq
\E[(R_{S'})_{(j:k)}\mid R_{1:N}\text{ with index }i\text{ removed}],
\end{align}
where $S'$ is a uniform size-$k$ subset of $[N]\setminus\{i\}$.  Then
\begin{equation}
\label{eq:app-batch-adv-def}
a_{i,j}\defeq q_{i,j}-v^{(-i)}_j.
\end{equation}
The formulas below are written in sorted-rank space and then mapped back to original indices.

\subsection{Include-one expectation}
\label{app:batch-include-one}

Let $t$ be the sorted rank of the datapoint forced into the subset.  Conditioning on including rank $t$ means $R_{(t:N)}$ is forced into the size-$k$ subset and the remaining $k-1$ indices are chosen uniformly from the other $N-1$ ranks.

\begin{lemma}[Include-one weights]
\label{lem:app-include-weights}
For fixed rank $j$ and included sorted rank $t$,
\[
\Prob((R_S)_{(j:k)} = R_{(m:N)} \mid t\in S,R_{1:N})
=
\begin{cases}
\dfrac{\binom{m-1}{j-1}\binom{N-m-1}{k-j-1}}{\binom{N-1}{k-1}}, & m<t,\\[1.0ex]
\dfrac{\binom{t-1}{j-1}\binom{N-t}{k-j}}{\binom{N-1}{k-1}}, & m=t,\\[1.0ex]
\dfrac{\binom{m-2}{j-2}\binom{N-m}{k-j}}{\binom{N-1}{k-1}}, & m>t.
\end{cases}
\]
\end{lemma}

\begin{proof}
We count size-$k$ subsets containing rank $t$ for which rank $m$ is the $j$th order statistic.  If $m<t$, then $m$ must be included, $j-1$ selected ranks come from the $m-1$ lower positions, and the remaining $k-j-1$ ranks come from positions above $m$ excluding the forced rank $t$.  If $m=t$, the forced rank is the $j$th order statistic, so choose $j-1$ lower and $k-j$ higher ranks.  If $m>t$, the forced rank is one of the $j-1$ lower ranks, so choose $j-2$ additional lower ranks and $k-j$ higher ranks.  Divide each count by the number $\binom{N-1}{k-1}$ of size-$k$ subsets containing $t$.
\end{proof}

Define the three precomputable matrices
\begin{align}
\label{eq:app-ABC-def}
A_{m,j}&\defeq
\frac{\binom{m-1}{j-1}\binom{N-m-1}{k-j-1}}{\binom{N-1}{k-1}},\\
B_{m,j}&\defeq
\frac{\binom{m-1}{j-1}\binom{N-m}{k-j}}{\binom{N-1}{k-1}},\\
C_{m,j}&\defeq
\frac{\binom{m-2}{j-2}\binom{N-m}{k-j}}{\binom{N-1}{k-1}}.
\end{align}
Then the include-one expectation for included sorted rank $t$ is
\begin{equation}
\label{eq:app-include-sum}
q_{t,j}
=
\sum_{m<t} R_{(m:N)}A_{m,j}
+R_{(t:N)}B_{t,j}
+\sum_{m>t} R_{(m:N)}C_{m,j}.
\end{equation}

\subsection{Leave-one-out expectation}
\label{app:batch-loo}

Remove sorted rank $t$ and sample a uniform size-$k$ subset from the remaining $N-1$ ranks.  Let
\begin{equation}
\label{eq:app-loo-weights}
W^-_{\ell,j}
\defeq
\frac{\binom{\ell-1}{j-1}\binom{(N-1)-\ell}{k-j}}{\binom{N-1}{k}},
\qquad \ell=1,\ldots,N-1.
\end{equation}
After removing rank $t$, the reduced sorted list satisfies
\[
R^{(-t)}_{(\ell:N-1)}=
\begin{cases}
R_{(\ell:N)}, & \ell<t,\\
R_{(\ell+1:N)}, & \ell\ge t.
\end{cases}
\]
Therefore
\begin{equation}
\label{eq:app-loo-sum}
v^{(-t)}_j
=
\sum_{m<t} R_{(m:N)}W^-_{m,j}
+
\sum_{m>t} R_{(m:N)}W^-_{m-1,j}.
\end{equation}

\subsection{Prefix/suffix implementation}
\label{app:batch-prefix-suffix}

Equations \eqref{eq:app-include-sum} and \eqref{eq:app-loo-sum} can be evaluated for all $t$ and $j$ in $O(Nk)$ time after sorting.

For include-one computation, let $M_R\defeq\mathbf{R}_{(\cdot)}\mathbf{1}_k^\top\in\R^{N\times k}$ and define
\[
P_A\defeq M_R\odot A,
\qquad
P_C\defeq M_R\odot C,
\]
where $\odot$ is the Hadamard product.  With cumulative sums down rows and row $0$ equal to zero,
\[
\sum_{m<t} R_{(m:N)}A_{m,:}
=
\operatorname{cumsum}(P_A)_{t-1,:},
\qquad
\sum_{m>t} R_{(m:N)}C_{m,:}
=
\operatorname{cumsum}(P_C)_{N,:}-\operatorname{cumsum}(P_C)_{t,:}.
\]
Thus all $q_{t,:}$ are obtained by two cumulative sums and rowwise additions.

For leave-one-out computation, let $u^-=(R_{(1:N)},\ldots,R_{(N-1:N)})^\top$ and $u^+=(R_{(2:N)},\ldots,R_{(N:N)})^\top$.  Form
\[
P_1\defeq u^-\mathbf{1}_k^\top\odot W^-,
\qquad
P_2\defeq u^+\mathbf{1}_k^\top\odot W^-.
\]
Then
\begin{equation}
\begin{aligned}
\sum_{m<t} R_{(m:N)}W^-_{m,:}
&=
\operatorname{cumsum}(P_1)_{t-1,:},\\
\sum_{m>t} R_{(m:N)}W^-_{m-1,:}
&=
\operatorname{cumsum}(P_2)_{N-1,:}-\operatorname{cumsum}(P_2)_{t-1,:}.
\end{aligned}
\end{equation}
This gives $v^{(-t)}$ for all ranks $t$ in $O(Nk)$ time.  Finally,
\begin{equation}
\label{eq:app-batch-adv}
a_{t,j}=q_{t,j}-v^{(-t)}_j,
\qquad
a_t^{(\alpha)}=\sum_{j=1}^k\alpha_j a_{t,j},
\end{equation}
and the inverse sorting permutation maps $a_{t,j}$ back to $a_{i,j}$.

\section{Advantage equivalence proof}
\label{app:advantage-equivalence-proof}

Let $I$ be uniformly random on $[N]$, independent of the dataset.  Fix an arm $b$.  The exact known-$(\bar R,p)$ advantage is $\bar a_{b,j}=\bar q_{b,j}-\bar v_j$.

\begin{theorem}[Batch advantage equivalence]
\label{thm:app-adv-equivalence}
Assume $1\le k<N$.  For every rank $j\in\{1,\ldots,k\}$,
\[
\E[a_{I,j}\mid A^{(I)}=b]=\bar a_{b,j}.
\]
Consequently, for any fixed $\alpha\in\R^k$,
\[
\E[a_I^{(\alpha)}\mid A^{(I)}=b]=\bar a_b^{(\alpha)}.
\]
\end{theorem}

\begin{proof}
We prove the include-one and leave-one-out identities separately and subtract.

\paragraph{Include-one term.}
Condition on $A^{(I)}=b$.  Then $R_I=\bar R_b$ is fixed, and the remaining $N-1$ batch rewards are still i.i.d. with the marginal law of $R=\bar R_A$ under $A\sim p$.  Given the dataset, $q_{I,j}$ is the expected $j$th order statistic of a size-$k$ subset conditioned to include index $I$.  Equivalently, it is the expected order statistic of the multiset consisting of the fixed reward $\bar R_b$ and a uniformly chosen $(k-1)$-subset of the remaining $N-1$ values.  By Lemma~\ref{lem:app-subsample-iid}, that random subset has the same unconditional distribution as $k-1$ fresh i.i.d. draws from the reward law.  Therefore
\[
\E[q_{I,j}\mid A^{(I)}=b]
=
\E[R_{(j:k)}\mid A_1=b]
=
\bar q_{b,j}.
\]

\paragraph{Leave-one-out term.}
Again condition on $A^{(I)}=b$.  Removing index $I$ leaves $N-1$ i.i.d. rewards with the law of $R$.  The leave-one-out quantity $v^{(-I)}_j$ is the U-statistic value estimator computed on those $N-1$ samples.  By Theorem~\ref{thm:app-Ustat-unbiased}, applied with $N-1$ samples,
\[
\E[v^{(-I)}_j\mid A^{(I)}=b]=\bar v_j.
\]

\paragraph{Subtract.}
Since $a_{I,j}=q_{I,j}-v^{(-I)}_j$,
\[
\E[a_{I,j}\mid A^{(I)}=b]
=
\bar q_{b,j}-\bar v_j
=
\bar a_{b,j}.
\]
Linearity gives the L-statistic statement.
\end{proof}

\section{Gradient estimators from batch quantities}
\label{app:gradient-estimators}

\subsection{Likelihood-ratio estimator}
\label{app:lr-estimator}

For $1\le k<N$, the sampling-regime likelihood-ratio estimator is
\begin{equation}
\label{eq:app-lr-estimator}
\estmr{g}_{\mathrm{LR\mbox{-}OG},N,\alpha}
=
\frac{k}{N}\sum_{i=1}^N
a_i^{(\alpha)}\nabla_\theta\log p_\theta(x_i).
\end{equation}
In the finite-action known-reward setting,
\begin{align}
\E[\estmr{g}_{\mathrm{LR\mbox{-}OG},N,\alpha}]
&=
k\,\E\!\left[
a_I^{(\alpha)}
\nabla_\theta\log p_\theta(A^{(I)})
\right]\\
&=
k\sum_{b=1}^m
p_\theta(b)
\E[a_I^{(\alpha)}\mid A^{(I)}=b]
\nabla_\theta\log p_\theta(b)\\
&=
k\sum_{b=1}^m
p_\theta(b)\bar a_b^{(\alpha)}
\nabla_\theta\log p_\theta(b)\\
&=
\nabla_\theta\bar v^\alpha(\theta),
\end{align}
where the last equality is \eqref{eq:app-known-lr-gradient}.  Thus the LR estimator uses the regular batch quantity $a_i^{(\alpha)}$; barred quantities appear only in the proof and in the exact known-regime formula.

\subsection{Reparameterization estimator}
\label{app:rp-estimator}

Suppose $x_i=T_\theta(\eps_i)$ with $\eps_i$ independent of $\theta$ and $R_i(\theta)=R(T_\theta(\eps_i))$ differentiable.  The reparameterization estimator differentiates the batch value
\begin{equation}
\label{eq:app-rp-estimator}
\estmr{g}_{\mathrm{RP\mbox{-}OG},N,\alpha}
=
\nabla_\theta v^\alpha(\theta),
\qquad
v^\alpha(\theta)=\mathbf{R}_{(\cdot)}(\theta)^\top W\alpha.
\end{equation}
For almost surely distinct rewards, the sorting permutation is locally constant, so
\begin{equation}
\label{eq:app-rp-sorted-gradient}
\estmr{g}_{\mathrm{RP\mbox{-}OG},N,\alpha}
=
\sum_{m=1}^N (W\alpha)_m\,\nabla_\theta R_{(m:N)}(\theta).
\end{equation}
With ties, one may use a stable-sort subgradient or a differentiable sorting relaxation.  Under dominated convergence,
\[
\E[\estmr{g}_{\mathrm{RP\mbox{-}OG},N,\alpha}]
=
\nabla_\theta\E[v^\alpha(\theta)]
=
\nabla_\theta\bar v^\alpha(\theta).
\]

\section{Computation, regularity, and ties}
\label{app:computation-regularity-ties}

\subsection{Collapsed L-statistic computation}
\label{app:collapsed-alpha-computation}

The prefix/suffix formulas in \app{app:batch-prefix-suffix} compute every rank-specific quantity $q_{t,j}$, $v_j^{(-t)}$, and $a_{t,j}$ in $O(Nk)$ time after sorting.  Often the algorithm only needs the L-statistic combination $a_t^{(\alpha)}=\sum_j\alpha_j a_{t,j}$.  In that case the rank dimension can be collapsed before applying the weights to a batch.

Define the collapsed vectors
\[
A_m^{(\alpha)}\defeq\sum_{j=1}^k\alpha_j A_{m,j},
\quad
B_m^{(\alpha)}\defeq\sum_{j=1}^k\alpha_j B_{m,j},
\quad
C_m^{(\alpha)}\defeq\sum_{j=1}^k\alpha_j C_{m,j},
\quad
(W^-)_\ell^{(\alpha)}\defeq\sum_{j=1}^k\alpha_j W^-_{\ell,j}.
\]
Then, in sorted-rank space,
\[
q_t^{(\alpha)}
=
\sum_{m<t}R_{(m:N)}A_m^{(\alpha)}
+
R_{(t:N)}B_t^{(\alpha)}
+
\sum_{m>t}R_{(m:N)}C_m^{(\alpha)},
\]
and
\[
(v^{(-t)})^{(\alpha)}
=
\sum_{m<t}R_{(m:N)}(W^-)_m^{(\alpha)}
+
\sum_{m>t}R_{(m:N)}(W^-)_{m-1}^{(\alpha)}.
\]
Thus all $a_t^{(\alpha)}=q_t^{(\alpha)}-(v^{(-t)})^{(\alpha)}$ can be computed with a constant number of prefix/suffix scans, i.e., in $O(N)$ time after sorting.  For fixed $N$, $k$, and $\alpha$, the collapsed weights can be precomputed once, so the per-minibatch cost is dominated by sorting and is $O(N\log N)$.  If one needs all rankwise advantages or changes $\alpha$ online, the relevant $O(Nk)$ precomputation or rankwise computation should be counted.

The same collapse applies to the reparameterization value:
\[
v^\alpha=\mathbf{R}_{(\cdot)}^\top W\alpha.
\]
For fixed $W\alpha$, evaluating $v^\alpha$ and its pathwise gradient after sorting is linear in $N$.

\subsection{Continuous rewards and ties}
\label{app:continuous-rewards-ties}

The finite-action deterministic-reward setting is useful because it gives exact barred quantities and a clean proof target for the LR estimator.  The sampling-regime definitions of $v_j$, $q_{i,j}$, $v_j^{(-i)}$, and $a_{i,j}$ are nevertheless purely batch-level definitions and remain well-defined for continuous rewards, trajectories, completions, or latent draws.

For continuous reward distributions, ties occur with probability zero under mild non-atomicity assumptions, and the sorted-reward formulas apply directly.  With atoms or equal numerical rewards, we use a stable total ordering of batch positions and do not collapse equal values.  The counting formulas in \app{app:batch-sort-weights}--\app{app:batch-loo} count sorted positions rather than distinct reward values, so the resulting expectations remain unchanged when tied positions have equal rewards.

For the likelihood-ratio estimator, the appendix proof states the finite-action version:
\[
\E[\estmr{g}_{\mathrm{LR\mbox{-}OG},N,\alpha}]
=
\nabla_\theta\bar v^\alpha(\theta).
\]
A more general trajectory-level statement follows the same score-function argument when the support is common in $\theta$, differentiation can be interchanged with integration, and regular conditional versions of the include-one values exist.  For the reparameterization estimator, the condition used in \app{app:rp-estimator} is the standard pathwise one: $R(T_\theta(\eps))$ is differentiable and dominated convergence permits exchanging $\nabla_\theta$ and $\E$.

\FloatBarrier
\section{Additional Order-Statistic Weight Schemes}
\label{app:ordergrad-alpha-schemes}

Here we visualize additional order-statistics weights used for
computing $v^\alpha$.  In this appendix only, plot labels of the form \texttt{Rank r} count ranks from the top: Rank 1 denotes the maximum, whereas the algebraic order-statistic index $j$ in the main text counts from the bottom, so $j=1$ denotes the minimum.  Weights near the right side of a plot emphasize
larger values, while weights near the left side emphasize smaller
values.  Some choices of $\alpha$, such as the Gini mean difference,
use signed weights to compare the two tails rather than to average a
subset of order statistics.

\begin{table}[H]
\centering
\caption{Summary of the order-statistic weight schemes visualized in this appendix.}
\label{tab:ordergrad-alpha-schemes}
\resizebox{0.95\linewidth}{!}{    \begin{tabular}{@{}p{0.22\linewidth}p{0.25\linewidth}p{0.46\linewidth}@{}}
        \toprule
        Scheme & Plot label / parameterization & Meaning \\
        \midrule
        Rank-$r$ & \texttt{Rank r} & Emphasizes the $r$th largest value among $k$ sampled values.  Rank 1 is the maximum; rank $k$ is the minimum. \\
        ReMax & \texttt{ReMax} with parameter $k$ & The maximum-of-$k$ objective, i.e., the Rank-1 scheme.  Increasing $k$ moves mass toward the largest order statistics; $k=1$ recovers the ordinary mean. \\
        Quantile & \texttt{Quantile:q} & Concentrates mass around the empirical $q$-quantile, with smaller $q$ focusing on lower sorted indices and $q=0.5$ corresponding to the median. \\
        TopM & \texttt{TopM:M} & Averages the top $M$ ranks out of a size-$k$ comparison batch. \\
        BotM & \texttt{BotM:M} & Averages the bottom $M$ ranks out of a size-$k$ comparison batch. \\
        TopBot & \texttt{TopBot:M} & Places mass on both tails by averaging the corresponding TopM and BotM schemes. \\
        Median & \texttt{Median} & Focuses on the central rank(s), providing a robust alternative to the mean. \\
        Trimmed mean & \texttt{TrimM:M} & Drops the bottom $M$ and top $M$ ranks and averages the remaining middle ranks. \\
        Winsorized mean & \texttt{WinsorizedM:M} & Clips the bottom and top $M$ ranks to the nearest retained ranks before averaging, reducing tail sensitivity without discarding the tails entirely. \\
        Gini mean difference & \texttt{GiniMeanDifference} & Uses signed weights to contrast high and low order statistics, emphasizing distributional spread rather than a location statistic. \\
        \bottomrule
    \end{tabular}
}
\end{table}

\providecommand{\ordergradweightfigwidth}{0.86\linewidth}

\begin{figure}[t]
    \centering
    \includegraphics[width=\ordergradweightfigwidth]{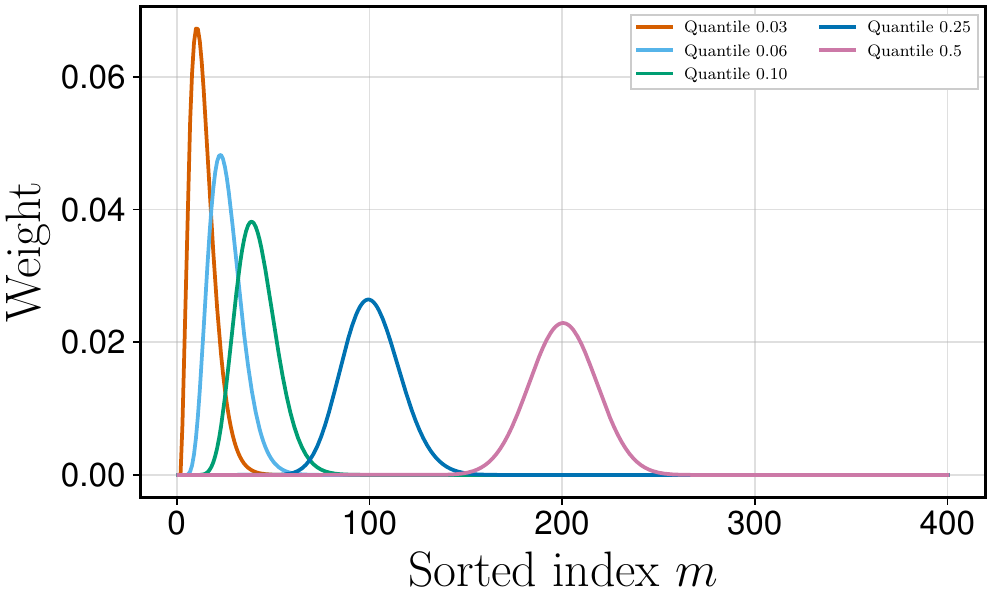}
    \caption{Quantile weight profiles for $q \in \{0.03,0.06,0.10,0.25,0.5\}$ with $N=400$ and comparison size $k=100$.  Smaller quantiles place mass on the lower tail of the sorted batch, while the median profile is centered near $m=N/2$.}
    \label{fig:ordergrad-weights-quantiles}
\end{figure}

\begin{figure}[t]
    \centering
    \includegraphics[width=\ordergradweightfigwidth]{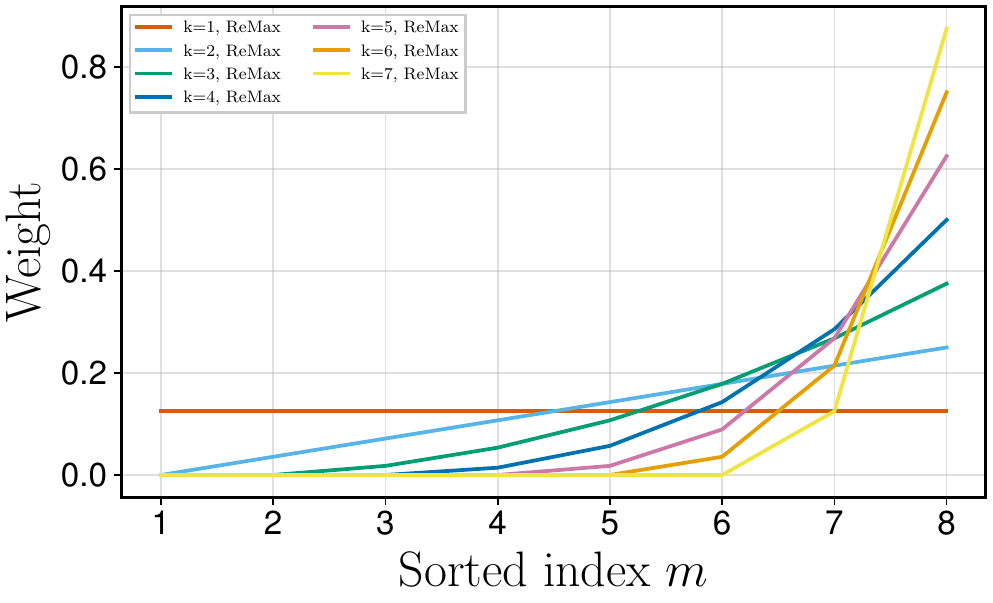}
    \caption{ReMax, or maximum-of-$k$, weight profiles for $N=8$ and $k \in \{1,\ldots,7\}$.  The $k=1$ curve is uniform, corresponding to the ordinary mean, while larger $k$ increasingly concentrates weight on the largest sorted values.}
    \label{fig:ordergrad-weights-remax}
\end{figure}

\begin{figure}[t]
    \centering
    \includegraphics[width=\ordergradweightfigwidth]{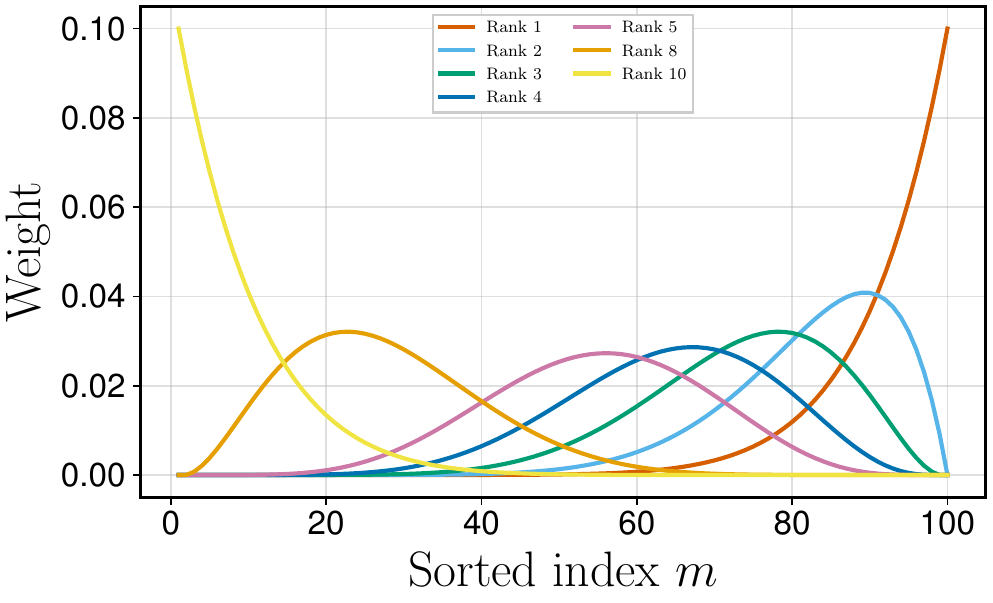}
    \caption{Rank-weight profiles for a comparison batch of size $k=10$ with $N=100$.  Rank 1 targets the maximum of the comparison batch; increasing the rank moves the mass toward lower order statistics, with Rank 10 targeting the minimum.}
    \label{fig:ordergrad-weights-ranks-k10}
\end{figure}

\begin{figure}[t]
    \centering
    \includegraphics[width=\ordergradweightfigwidth]{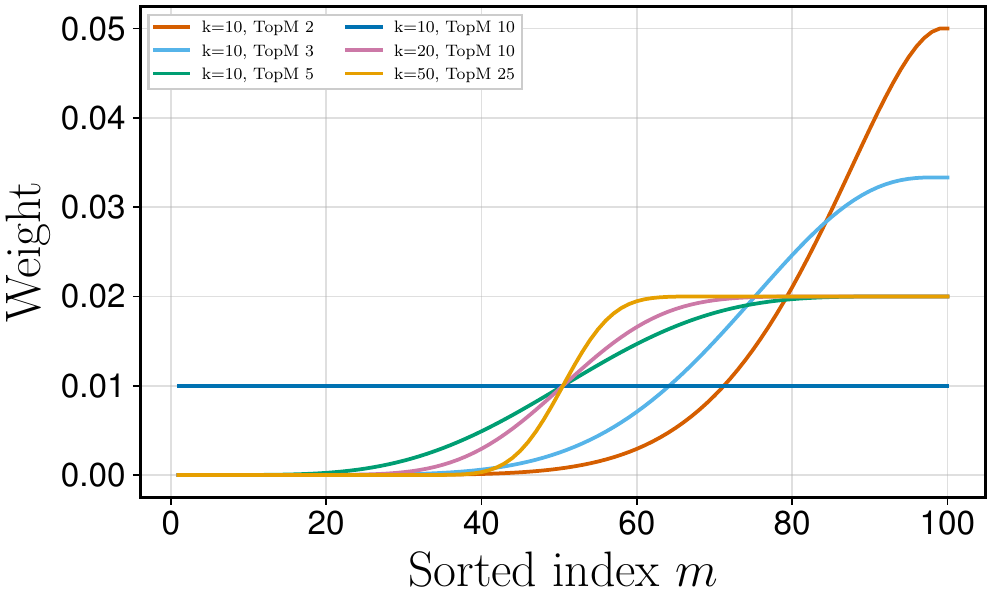}
    \caption{TopM profiles for several choices of $M$ and $k$ with $N=100$.  These schemes interpolate between a strongly top-focused objective and the ordinary mean: when $M=k$, all ranks are averaged and the resulting profile is uniform.}
    \label{fig:ordergrad-weights-topm}
\end{figure}

\begin{figure}[t]
    \centering
    \includegraphics[width=\ordergradweightfigwidth]{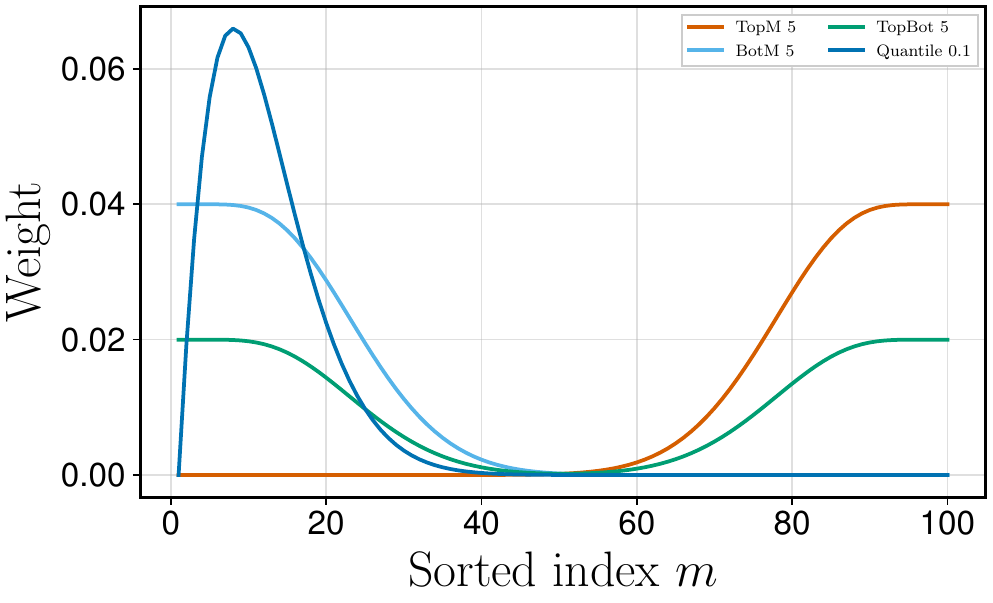}
    \caption{Tail- and quantile-focused schemes for $N=100$ and $k=20$.  TopM emphasizes high values, BotM emphasizes low values, TopBot places mass on both tails, and the quantile scheme concentrates around the specified lower quantile.}
    \label{fig:ordergrad-weights-various}
\end{figure}

\begin{figure}[t]
    \centering
    \includegraphics[width=\ordergradweightfigwidth]{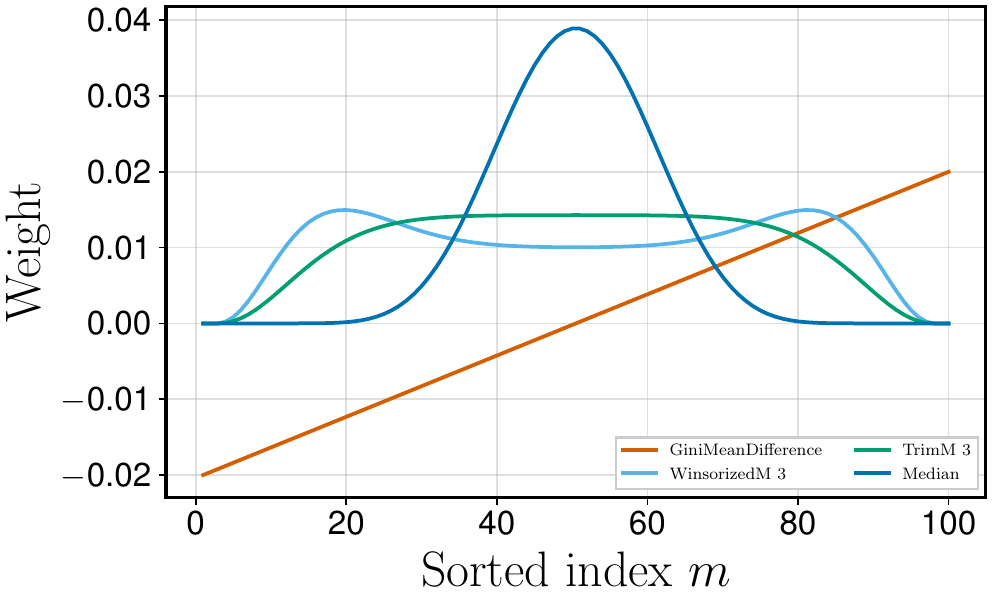}
    \caption{Robust and signed schemes for $N=100$ and $k=20$.  The median focuses on the center, the trimmed and winsorized means reduce sensitivity to extremes, and the Gini mean difference uses signed weights to contrast the upper and lower tails.}
    \label{fig:ordergrad-weights-various-robust}
\end{figure}

\FloatBarrier

\section{Additional Toy Illustrative Experiments}

\subsection{Illustrative Experiment on Portfolio Management using CVaR}

\paragraph{High-yield tail-risk example.}
We include a small synthetic trading example to make the effect of lower-tail
training visually explicit.  The experiment uses the static
\texttt{synthetic\_tailrisk} environment.  Each episode contains three risky
assets and a cash asset.  The risky assets are named
\[
    \texttt{high\_yield\_tailrisk},\qquad
    \texttt{defensive\_hedge},\qquad
    \texttt{balanced}.
\]
The cash asset has zero return.  Its portfolio weight is capped at $10\%$, so
that a lower-tail objective cannot solve the problem by moving entirely to cash.
Policies are long-only and fully invested:
\[
    w_{t,a}\geq 0,\qquad
    \sum_a w_{t,a}=1,\qquad
    w_{t,\mathrm{cash}}\leq 0.10.
\]

The risky-asset returns are generated by a simple one-factor negative-skew
model.  In normal periods,
\[
    r_{t,a}
    =
    \frac{\mu_a}{252}
    + b_a f_t
    + \frac{\sigma_a}{\sqrt{252}}\epsilon_{t,a},
    \qquad
    f_t\sim \mathcal{N}(0,0.004^2),
    \qquad
    \epsilon_{t,a}\sim \mathcal{N}(0,1),
\]
with annualized drift parameters
\[
    \mu=(0.70,\;0.030,\;0.095),
\]
annualized volatility parameters
\[
    \sigma=(0.11,\;0.065,\;0.075),
\]
and factor loadings
\[
    b=(0.80,\;-0.10,\;0.25).
\]
Thus, the first asset has very high normal-period drift and positive factor
exposure, the second asset is a defensive hedge with negative factor loading,
and the third asset has intermediate behavior.

With probability $p$ per episode, an unobserved disaster occurs at a random
time.  If an event occurs, the event day is sampled uniformly from the eligible
post-lookback portion of the simulated tensor.  In the reported horizons this
corresponds to total-index days $20$--$97$ in the $100$-day training tensor and
days $36$--$177$ in the $180$-day deployment tensor, where the first $20$ days
are used as the lookback window.  At the event, the high-yield asset receives a
large negative jump,
\[
    J\sim \mathrm{Unif}(0.50,0.70),
    \qquad
    r_{\tau,\mathrm{tail}}\leftarrow r_{\tau,\mathrm{tail}}-J.
\]
The hedge receives a positive jump,
\[
    H\sim \mathrm{Unif}(0.12,0.26),
    \qquad
    r_{\tau,\mathrm{hedge}}\leftarrow r_{\tau,\mathrm{hedge}}+H,
\]
and the balanced asset receives a smaller loss,
\[
    B\sim \mathrm{Unif}(0.015,0.050),
    \qquad
    r_{\tau,\mathrm{balanced}}\leftarrow r_{\tau,\mathrm{balanced}}-B.
\]
Two aftershock days follow when they remain inside the episode.  On the next two
days, with fractions $c_1=0.28$ and $c_2=0.14$, the high-yield asset loses
$c_\ell J$, the hedge gains $0.30c_\ell J$, and the balanced asset loses
$0.07c_\ell J$.  All simple returns are clipped to the interval
$[-0.82,0.35]$.

The training distribution uses
\[
    p_{\mathrm{train}}=0.10,
\]
whereas the main deployment distribution used for
Figure~\ref{fig:tailrisk-toy} and Table~\ref{tab:tailrisk-toy-metrics} uses
\[
    p_{\mathrm{test}}=0.35.
\]
The launch script also evaluates a post-training stress grid
\[
    p\in\{0.00,0.10,0.20,0.35,0.50,0.65\},
\]
but the compact publication figure and table report the main deployment setting
$p_{\mathrm{test}}=0.35$.  No contiguous crisis block is used in this static
tail-risk figure; the deployment shift is the higher per-episode disaster
probability.

\paragraph{Policy, information set, and wealth dynamics.}
The policy observes only causal features.  At each rebalance date, its input
contains the trailing mean and volatility of each risky asset over a $20$-day
lookback window, the previous portfolio weights, the most recent realized
portfolio return, an exponential moving average of portfolio return, current
drawdown, and normalized time.  The rolling means are multiplied by $252$, the
rolling volatilities by $\sqrt{252}$, and the last and exponential-moving-average
portfolio returns by $25$.  With three risky assets and cash, this gives a
$14$-dimensional feature vector:
\[
    2\times 3 \text{ asset moments}
    +4 \text{ previous weights}
    +4 \text{ scalar state variables}.
\]
The policy does not observe the disaster indicator or the future disaster time.

We use the default linear policy, so the feature vector is mapped directly to
four portfolio logits, one for each risky asset and one for cash.  The final
linear layer is initialized at zero, so training starts from equal logits rather
than from an arbitrary corner.  During training, Gaussian action noise with
standard deviation $0.35$ is added to the logits.  During deployment, the
deterministic mean logits are used.

If $\tilde w_t=\mathrm{softmax}(z_t)$ denotes the raw softmax weights, the cash
cap is implemented as
\[
    w_{t,\mathrm{cash}}=0.10\,\tilde w_{t,\mathrm{cash}},
    \qquad
    w_{t,a}
    =
    (1-w_{t,\mathrm{cash}})
    \frac{\tilde w_{t,a}}
    {\sum_{b\in\mathcal{R}}\tilde w_{t,b}},
    \qquad a\in\mathcal{R},
\]
where $\mathcal{R}$ is the set of risky assets.  Thus, the final cash weight is
always at most $10\%$, and the remaining mass is renormalized across risky
assets.

Training episodes have $80$ traded days and deployment rollouts have $160$
traded days.  The first $20$ simulated days are used only to form the initial
lookback window.  Portfolios are rebalanced every $5$ traded days.  The same
weight is held between rebalances.  Proportional transaction costs are charged
on the first day of each rebalance block with coefficient
\[
    \lambda=0.001.
\]
If $w_t$ is the new weight vector and $w_{t^-}$ is the previous weight vector,
the first day of the block uses
\[
    W_{t+1}
    =
    W_t\left(1+w_t^\top r_{t+1}
    -\lambda\|w_t-w_{t^-}\|_1\right),
\]
and subsequent days in the same block use the same portfolio weight without an
additional rebalance cost.  The episode score is terminal log wealth,
\[
    Z_i(\theta)=\log W_{T,i}(\theta).
\]

\paragraph{Methods.}
Both methods use the same simulator, policy class, optimizer, random-seed
protocol, minibatch size, and vectorized likelihood-ratio estimator.  Each
training iteration simulates a fresh minibatch of synthetic episodes.  The
optimizer is Adam with learning rate $0.003$ and gradient clipping threshold
$1.0$.  Both methods are trained for $750$ optimization steps with minibatch
size
\[
    B=256.
\]

The REINFORCE baseline optimizes mean terminal log wealth.  In the implementation
the terminal log-wealth rewards are centered and normalized within the minibatch,
\[
    A_i^{\mathrm{mean}}
    =
    \frac{Z_i-\bar Z}{\widehat{\sigma}_Z+\varepsilon},
    \qquad
    \bar Z=\frac{1}{B}\sum_{j=1}^B Z_j,
\]
and the policy-gradient loss is
\[
    -\frac{1}{B}\sum_{i=1}^B
    A_i^{\mathrm{mean}}
    \sum_t \log \pi_\theta(a_{i,t}\mid s_{i,t}).
\]
The normalization changes the scale of the gradient estimate but not the
mean-return objective being targeted.

The robust variant replaces the mean-return advantage by an OrderGrad advantage
for a lower-tail L-statistic.  Let
$Z_{(1)}\leq \cdots \leq Z_{(B)}$ denote the sorted terminal log wealths.  The
lower-tail target is
\[
    L_\alpha(Z)
    =
    \frac{1}{\lceil \alpha B\rceil}
    \sum_{j=1}^{\lceil \alpha B\rceil} Z_{(j)} .
\]
In the reported run,
\[
    \alpha=0.20,\qquad K=128,
\]
where $K$ is the OrderGrad interpolation parameter.  The implementation computes
the expected OrderGrad advantage for the statistic
\texttt{LowerTailMean:0.20} and then normalizes the resulting advantages by
their minibatch standard deviation before applying the same REINFORCE
likelihood-ratio loss.

The reported static deployment results aggregate $1280$ deterministic
deployment paths per method.  Since each seed uses $256$ deployment rollouts,
this corresponds to five training seeds.  The deployment evaluation is
deterministic: no Gaussian action noise is added to policy logits at test time.

\begin{figure}[t]
    \centering
        \includegraphics[width=0.9\textwidth]{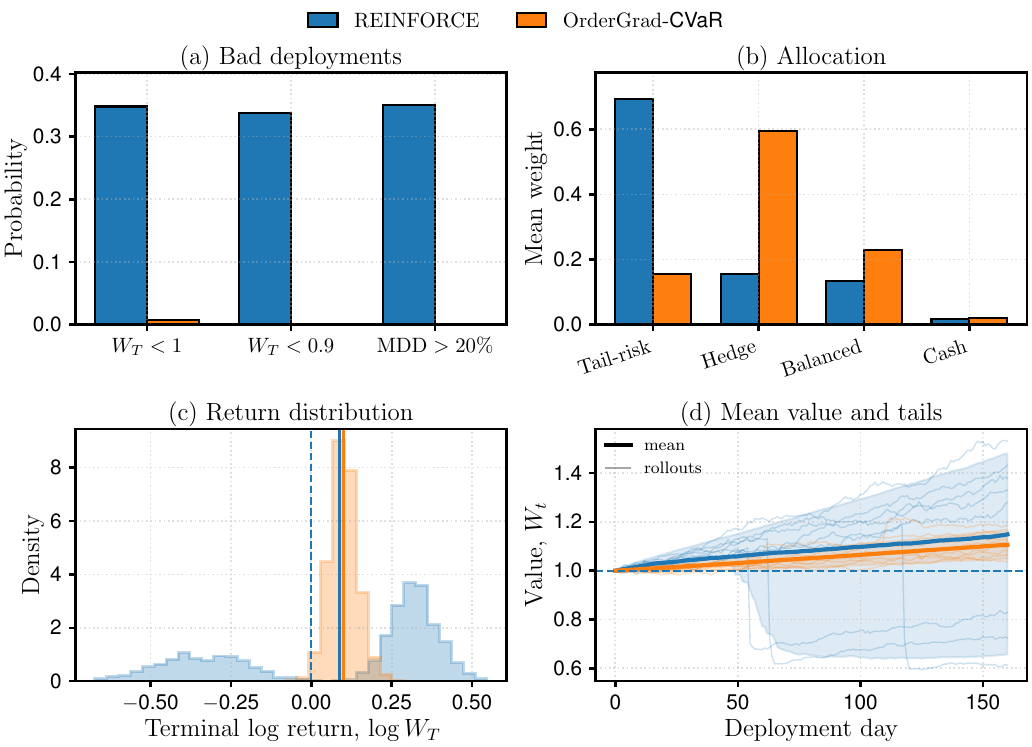}
    \caption{High-yield tail-risk trading example.  Panel (a) reports bad
    deployment probabilities: losing money, ending below $W_T<0.9$, and
    suffering a maximum drawdown larger than $20\%$ in magnitude.  Panel (b)
    reports average deployment portfolio weights.  Panel (c) shows the
    distribution of terminal log returns.  Panel (d) shows the mean deployment
    wealth, a $10$--$90\%$ band, and representative individual rollouts.
    Standard REINFORCE attains a higher mean cumulative return by concentrating
    on the high-yield tail-risk asset, but this produces a broad left tail and
    many bad deployments.  OrderGrad-CVaR reallocates toward the hedge and
    balanced assets, giving lower mean upside but substantially higher
    reliability.}
    \label{fig:tailrisk-toy}
\end{figure}

\paragraph{Result.}
Figure~\ref{fig:tailrisk-toy} shows the main effect.  Standard REINFORCE places
most of its capital in the high-yield tail-risk asset, whereas OrderGrad-CVaR
substantially reduces this exposure and reallocates toward the hedge and
balanced assets.  In the plotted run, the average high-yield exposure falls from
$70.2\%$ for REINFORCE to $16.0\%$ for OrderGrad-CVaR, while the hedge
allocation increases from $15.1\%$ to $59.2\%$.  The cash weights are small for
both methods, $1.6\%$ for REINFORCE and $2.2\%$ for OrderGrad-CVaR, so the
robust method is not winning by simply moving to cash.

This allocation shift produces a pronounced reliability difference.  REINFORCE
has the larger mean cumulative return, $14.6\%$ versus $10.7\%$ for
OrderGrad-CVaR, because it benefits from high-yield exposure in favorable
rollouts.  However, the deployment distribution is much less stable: REINFORCE
loses money in $35.3\%$ of deployment paths, falls below terminal wealth
$W_T<0.9$ in $34.2\%$ of paths, and experiences a maximum drawdown larger than
$20\%$ in magnitude in $35.4\%$ of paths.  OrderGrad-CVaR loses money in only
$0.47\%$ of paths, has no paths below $W_T<0.9$ in this run, and has no paths
exceeding the $20\%$ drawdown threshold.  The lower-tail statistics show the
same pattern: the $10$th percentile terminal log return improves from $-0.429$
for REINFORCE to $0.046$ for OrderGrad-CVaR, while mean maximum drawdown
improves from $-18.9\%$ to $-2.2\%$.  The daily $5\%$ CVaR also improves from
$-3.21\%$ to $-0.57\%$.

\begin{table}[t]
    \centering
    \caption{Deployment metrics in the high-yield tail-risk example.  Higher is
    better for mean return, profitable
    paths.  Lower is better for bad-outcome probabilities and
    drawdown severity.}
    \label{tab:tailrisk-toy-metrics}
    \begin{tabular}{lrrrrrr}
        \toprule
        Method
        & Mean return
        & Profitable
        & $W_T<0.9$
        & MDD $>20\%$
        & Mean MDD \\
        \midrule
        REINFORCE
        & $14.6\%$
        & $64.7\%$
        & $34.2\%$
        & $35.4\%$
        & $-18.9\%$ \\
        OrderGrad-CVaR
        & $10.7\%$
        & $99.5\%$
        & $0.0\%$
        & $0.0\%$
        & $-2.2\%$ \\
        \bottomrule
    \end{tabular}
\end{table}

\paragraph{Interpretation.}
The experiment is deliberately simple: it is designed to isolate the effect of
the training objective rather than to model a complete financial market.  The
result should therefore not be read as evidence that CVaR training always gives
higher average return.  Instead, it demonstrates the expected risk-sensitive
behavior in a controlled negative-skew allocation problem.  The mean-return
policy accepts the high-yield tail-risk trade because its ordinary outcomes are
attractive.  The CVaR-trained policy gives more weight to the bad episodes
observed during training and therefore sacrifices some average upside in
exchange for a substantially better lower tail.  This is the behavior desired
when the deployment criterion values reliability, drawdown control, or avoidance
of rare large losses more than mean return alone.

\FloatBarrier

\subsection{Robust regression with an order-statistic training objective}
\label{app:robust-regression}

We evaluate a synthetic robust-regression problem designed to test whether an order-statistic objective can suppress a coherent set of corrupted labels.  Clean examples are generated from
\[
    y = 2.0x -1.0 + \epsilon, \qquad \epsilon \sim \mathcal{N}(0,0.25^2),
\]
with \(x\sim\mathrm{Unif}[-3,3]\).  A fraction \(0.12\) of the training labels is replaced by samples from a wrong linear relation,
\[
    y = -1.4x +4.2 + \epsilon_\mathrm{out}, \qquad \epsilon_\mathrm{out} \sim \mathcal{N}(0,0.45^2).
\]
The model is the linear predictor \(\hat y = wx+b\).  We compare the ordinary mean reward objective, the median reward objective, and a trimmed L-statistic objective.  The per-example reward is minus squared error, so the trimmed objective emphasizes examples that are well fit by the current model and downweights both tails of the reward order statistics.  In the reported run, the optimizer uses minibatches of size \(256\), \(180\) Adam steps, and the trimmed objective \(\mathrm{TrimM}:4\).

Table~\ref{tab:robust-regression} reports clean test error averaged over \(10\) random seeds.  The best clean MSE is obtained by Median with mean final MSE \(0.0641\).  This demonstrates that the order-statistic objective can recover a predictor closer to the uncontaminated data-generating process than objectives that average over the corrupted population.

\begin{table}[t]
\centering
\caption{Robust-regression summary.  Clean test MSE and parameter errors are averaged over random seeds; SE denotes the standard error of the final clean MSE.}
\begin{tabular}{lrrrr}
\toprule
Method & Final clean MSE & SE & $|\hat w - w_\star|$ & $|\hat b - b_\star|$ \\
\midrule
Median & 0.0641 & 0.00112 & 0.0139 & 0.0189 \\
Trimmed (m=4) & 0.0656 & 0.00107 & 0.0237 & 0.0226 \\
Mean & 1.1 & 0.101 & 0.433 & 0.667 \\
\bottomrule
\end{tabular}
\label{tab:robust-regression}
\end{table}

\begin{figure}[t]
\centering
\includegraphics[width=0.9\linewidth]{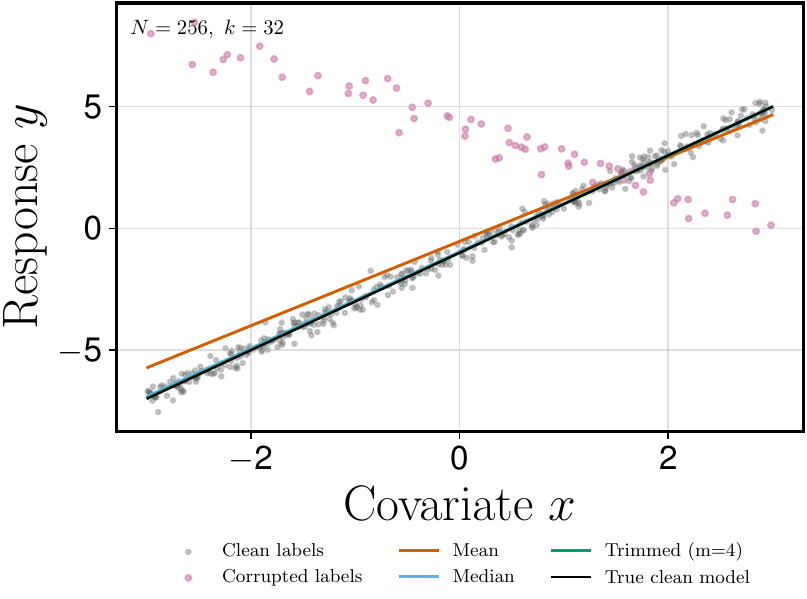}
\caption{Representative robust-regression fit panel.}
\label{fig:robust-regression-fit}
\end{figure}

 Figure~\ref{fig:robust-regression-fit} shows the fitted lines for one representative training set, illustrating that the trimmed objective follows the clean population line while the mean objective is pulled toward the corrupted population

\FloatBarrier

\section{LLM Experiments}
\label{app:llm}

In this section, we provide additional details and results for the LLM
experiments in Sec.~\ref{sec:reasoning_experiments}.

\subsection{Setting}
\label{app:llm:hparams}

All models are trained with a learning rate of $10^{-6}$, a batch size of
$1024$, and a PPO mini-batch size of $256$. For each input problem, we roll
out $8$ responses using a temperature of $1.0$. Both backbones are trained
on the MATH dataset~\citep{hendrycks2021measuring}. OrderGrad is applied
prompt-conditionally: for a fixed prompt $c$, the rollout group is treated as
samples from $p_\theta(\cdot\mid c)$ and rank advantages are computed within
that group. For all methods, the advantage is normalized by the group standard
deviation; this normalization, together with PPO clipping, is a practical
stabilizer and is not part of the exact unbiasedness statement in the theory
sections. \maxpo{} is swept over $K\in\{2,4,6\}$ and our method over
$(m,K)\in\{(1,2),(2,4),(2,6),(3,6)\}$. The prompt template used is presented below.

\begin{examplebox}{Prompt Template (Train)}
<|im_start|>system
You are a helpful assistant.<|im_end|>
<|im_start|>user
{input}
Please reason step by step, and put your final answer within \boxed{}.<|im_end|>
<|im_start|>assistant

\end{examplebox}

\begin{examplebox}{Prompt Template (Test)}
<|im_start|>system
Please reason step by step, and put your final answer within \boxed{}.<|im_end|>
<|im_start|>user
{input}<|im_end|>
<|im_start|>assistant
\end{examplebox}

\subsection{Additional Results}
\label{app:llm:additional_results}

Tables~\ref{tab:taskavg_qwen} and~\ref{tab:taskavg_qwen3} report the
task-average pass@$k$ on Qwen2.5-Math-7B and Qwen3-4B-Base, respectively.
Table~\ref{tab:perbench_pass1_pass256} reports the per-benchmark pass@$1$
and pass@$256$. Figures~\ref{fig:per_task_qwen}
and~\ref{fig:per_task_qwen3} provide the per-benchmark pass@$k$ curves.

\begin{table}[H]
\centering
\caption{Task-average pass@$k$ on Qwen2.5-Math-7B.}
\label{tab:taskavg_qwen}
\resizebox{0.95\linewidth}{!}{\begin{tabular}{lccccccccc}
\toprule
Method & $k{=}1$ & $k{=}2$ & $k{=}4$ & $k{=}8$ & $k{=}16$ & $k{=}32$ & $k{=}64$ & $k{=}128$ & $k{=}256$ \\
\midrule
base & 0.224 & 0.319 & 0.415 & 0.502 & 0.575 & 0.633 & 0.682 & 0.728 & 0.772 \\
GRPO & \textbf{0.423} & 0.485 & 0.537 & 0.585 & 0.629 & 0.669 & 0.706 & 0.742 & 0.776 \\
\maxpo{} ($K{=}2$) & 0.419 & \textbf{0.491} & \textbf{0.550} & \textbf{0.599} & 0.643 & 0.682 & 0.717 & 0.749 & 0.779 \\
\maxpo{} ($K{=}4$) & 0.388 & 0.466 & 0.531 & 0.587 & 0.635 & 0.678 & 0.719 & 0.760 & \textbf{0.802} \\
\maxpo{} ($K{=}6$) & 0.331 & 0.427 & 0.507 & 0.573 & 0.626 & 0.672 & 0.715 & 0.755 & 0.793 \\
Ours (Top$1$@$2$) & 0.422 & 0.491 & 0.547 & 0.598 & 0.642 & 0.683 & 0.720 & 0.756 & 0.792 \\
Ours (Top$2$@$4$) & 0.406 & 0.481 & 0.544 & 0.598 & \textbf{0.643} & 0.683 & 0.720 & 0.757 & 0.798 \\
Ours (Top$2$@$6$) & 0.399 & 0.476 & 0.539 & 0.594 & 0.641 & \textbf{0.683} & \textbf{0.723} & \textbf{0.762} & 0.802 \\
Ours (Top$3$@$6$) & 0.393 & 0.469 & 0.530 & 0.582 & 0.630 & 0.674 & 0.715 & 0.754 & 0.790 \\
\bottomrule
\end{tabular}}
\end{table}

\begin{table}[H]
\centering
\caption{Task-average pass@$k$ on Qwen3-4B-Base.}
\label{tab:taskavg_qwen3}
\resizebox{0.95\linewidth}{!}{\begin{tabular}{lccccccccc}
\toprule
Method & $k{=}1$ & $k{=}2$ & $k{=}4$ & $k{=}8$ & $k{=}16$ & $k{=}32$ & $k{=}64$ & $k{=}128$ & $k{=}256$ \\
\midrule
base & 0.310 & 0.399 & 0.475 & 0.540 & 0.594 & 0.643 & 0.688 & 0.729 & 0.764 \\
GRPO & 0.411 & 0.466 & 0.509 & 0.544 & 0.578 & 0.613 & 0.646 & 0.680 & 0.719 \\
\maxpo{} ($K{=}2$) & 0.403 & 0.474 & 0.531 & 0.581 & 0.628 & 0.674 & 0.714 & 0.749 & 0.781 \\
\maxpo{} ($K{=}4$) & 0.359 & 0.453 & 0.522 & 0.575 & 0.621 & 0.665 & 0.710 & 0.750 & 0.785 \\
\maxpo{} ($K{=}6$) & 0.319 & 0.422 & 0.500 & 0.560 & 0.610 & 0.655 & 0.696 & 0.734 & 0.768 \\
Ours (Top$1$@$2$) & 0.389 & 0.460 & 0.520 & 0.575 & 0.627 & 0.675 & 0.715 & 0.751 & 0.785 \\
Ours (Top$2$@$4$) & \textbf{0.415} & \textbf{0.494} & \textbf{0.555} & \textbf{0.607} & \textbf{0.656} & \textbf{0.702} & \textbf{0.742} & \textbf{0.778} & 0.810 \\
Ours (Top$2$@$6$) & 0.370 & 0.461 & 0.531 & 0.589 & 0.642 & 0.692 & 0.738 & 0.777 & \textbf{0.812} \\
Ours (Top$3$@$6$) & 0.378 & 0.475 & 0.545 & 0.601 & 0.652 & 0.698 & 0.739 & 0.772 & 0.802 \\
\bottomrule
\end{tabular}}
\end{table}

\begin{table}[H]
\centering
\caption{Per-benchmark pass@$1$ / pass@$256$ for both backbones.}
\label{tab:perbench_pass1_pass256}
\resizebox{0.95\linewidth}{!}{\setlength{\tabcolsep}{4pt}
\begin{tabular}{lcccccc}
\toprule
Method & AIME24 & AIME25 & AMC23 & MATH500 & Minerva & Avg. \\
\midrule
\multicolumn{7}{l}{\textit{Qwen2.5-Math-7B}}\\
\midrule
base & 0.118 / 0.705 & 0.051 / 0.466 & 0.365 / 0.990 & 0.468 / 0.959 & 0.117 / 0.741 & 0.224 / 0.772 \\
GRPO & 0.298 / 0.704 & 0.090 / 0.488 & 0.617 / 0.988 & 0.747 / 0.947 & \textbf{0.361} / 0.753 & \textbf{0.423} / 0.776 \\
\maxpo{} ($K{=}2$) & \textbf{0.307} / 0.723 & \textbf{0.105} / 0.465 & \textbf{0.637} / 0.980 & \textbf{0.747} / 0.958 & 0.300 / 0.771 & 0.419 / 0.779 \\
\maxpo{} ($K{=}4$) & 0.267 / 0.743 & 0.084 / \textbf{0.557} & 0.607 / 0.982 & 0.728 / 0.958 & 0.254 / 0.772 & 0.388 / \textbf{0.802} \\
\maxpo{} ($K{=}6$) & 0.208 / 0.717 & 0.075 / 0.521 & 0.537 / \textbf{0.998} & 0.645 / \textbf{0.959} & 0.193 / 0.768 & 0.331 / 0.793 \\
Ours (Top$1$@$2$) & 0.291 / \textbf{0.759} & 0.093 / 0.517 & 0.629 / 0.974 & \textbf{0.747} / 0.952 & 0.351 / 0.759 & 0.422 / 0.792 \\
Ours (Top$2$@$4$) & 0.283 / 0.758 & 0.095 / 0.521 & 0.624 / 0.983 & 0.739 / 0.957 & 0.290 / 0.772 & 0.406 / 0.798 \\
Ours (Top$2$@$6$) & 0.276 / 0.756 & 0.085 / 0.521 & 0.613 / 0.985 & 0.727 / 0.956 & 0.292 / \textbf{0.789} & 0.399 / 0.802 \\
Ours (Top$3$@$6$) & 0.268 / 0.741 & 0.089 / 0.482 & 0.607 / 0.990 & 0.723 / 0.955 & 0.278 / 0.783 & 0.393 / 0.790 \\
\midrule
\multicolumn{7}{l}{\textit{Qwen3-4B-Base}}\\
\midrule
base & 0.102 / 0.606 & 0.072 / 0.472 & 0.454 / 0.990 & 0.674 / 0.961 & 0.249 / 0.792 & 0.310 / 0.764 \\
GRPO & 0.164 / 0.604 & 0.105 / 0.377 & 0.583 / 0.966 & \textbf{0.788} / 0.933 & \textbf{0.416} / 0.713 & 0.411 / 0.719 \\
\maxpo{} ($K{=}2$) & 0.165 / 0.660 & 0.103 / 0.519 & 0.584 / 0.996 & 0.780 / 0.960 & 0.379 / 0.772 & 0.403 / 0.781 \\
\maxpo{} ($K{=}4$) & 0.134 / 0.648 & 0.122 / 0.564 & 0.504 / 0.991 & 0.713 / 0.967 & 0.319 / 0.753 & 0.359 / 0.785 \\
\maxpo{} ($K{=}6$) & 0.108 / 0.589 & 0.094 / 0.548 & 0.434 / 0.976 & 0.664 / 0.966 & 0.295 / 0.762 & 0.319 / 0.768 \\
Ours (Top$1$@$2$) & 0.148 / 0.656 & 0.092 / 0.563 & 0.571 / \textbf{0.996} & 0.771 / 0.958 & 0.364 / 0.753 & 0.389 / 0.785 \\
Ours (Top$2$@$4$) & \textbf{0.189} / \textbf{0.760} & \textbf{0.151} / 0.563 & \textbf{0.590} / 0.989 & 0.783 / \textbf{0.968} & 0.364 / 0.773 & \textbf{0.415} / 0.810 \\
Ours (Top$2$@$6$) & 0.160 / 0.747 & 0.097 / \textbf{0.574} & 0.538 / 0.992 & 0.731 / 0.967 & 0.323 / \textbf{0.778} & 0.370 / \textbf{0.812} \\
Ours (Top$3$@$6$) & 0.170 / 0.732 & 0.123 / 0.557 & 0.532 / 0.991 & 0.735 / 0.964 & 0.330 / 0.766 & 0.378 / 0.802 \\
\bottomrule
\end{tabular}}
\end{table}

\begin{table}[H]
\centering
\caption{Task-average pass@$k$ on Qwen2.5-Math-7B with a response-length reward.}
\label{tab:length_qwen}
\small
\setlength{\tabcolsep}{4pt}
\scalebox{0.9}{\begin{tabular}{lccccccccccc}
\toprule
Method & $k{=}1$ & $k{=}2$ & $k{=}4$ & $k{=}8$ & $k{=}16$ & $k{=}32$ & $k{=}64$ & $k{=}128$ & $k{=}256$ & $k{=}512$ & $k{=}1024$ \\
\midrule
Base & 0.224 & 0.319 & 0.415 & 0.502 & 0.575 & 0.633 & 0.682 & 0.728 & 0.772 & 0.817 & 0.859 \\
GRPO & 0.423 & 0.485 & 0.537 & 0.585 & 0.629 & 0.669 & 0.706 & 0.742 & 0.776 & 0.809 & 0.841 \\
Ours (Top2@4) & 0.406 & 0.481 & 0.544 & 0.598 & 0.643 & 0.683 & 0.720 & 0.757 & 0.798 & 0.843 & 0.887 \\
GRPO w/ length penalty & 0.157 & 0.191 & 0.222 & 0.249 & 0.270 & 0.289 & 0.305 & 0.318 & 0.331 & 0.345 & 0.363 \\
Ours (Top2@4, Bottom2@4) & 0.403 & 0.470 & 0.525 & 0.571 & 0.614 & 0.656 & 0.698 & 0.738 & 0.778 & 0.821 & 0.867 \\
\bottomrule
\end{tabular}}
\end{table}

\begin{figure}[H]
  \centering
  \includegraphics[width=\linewidth]{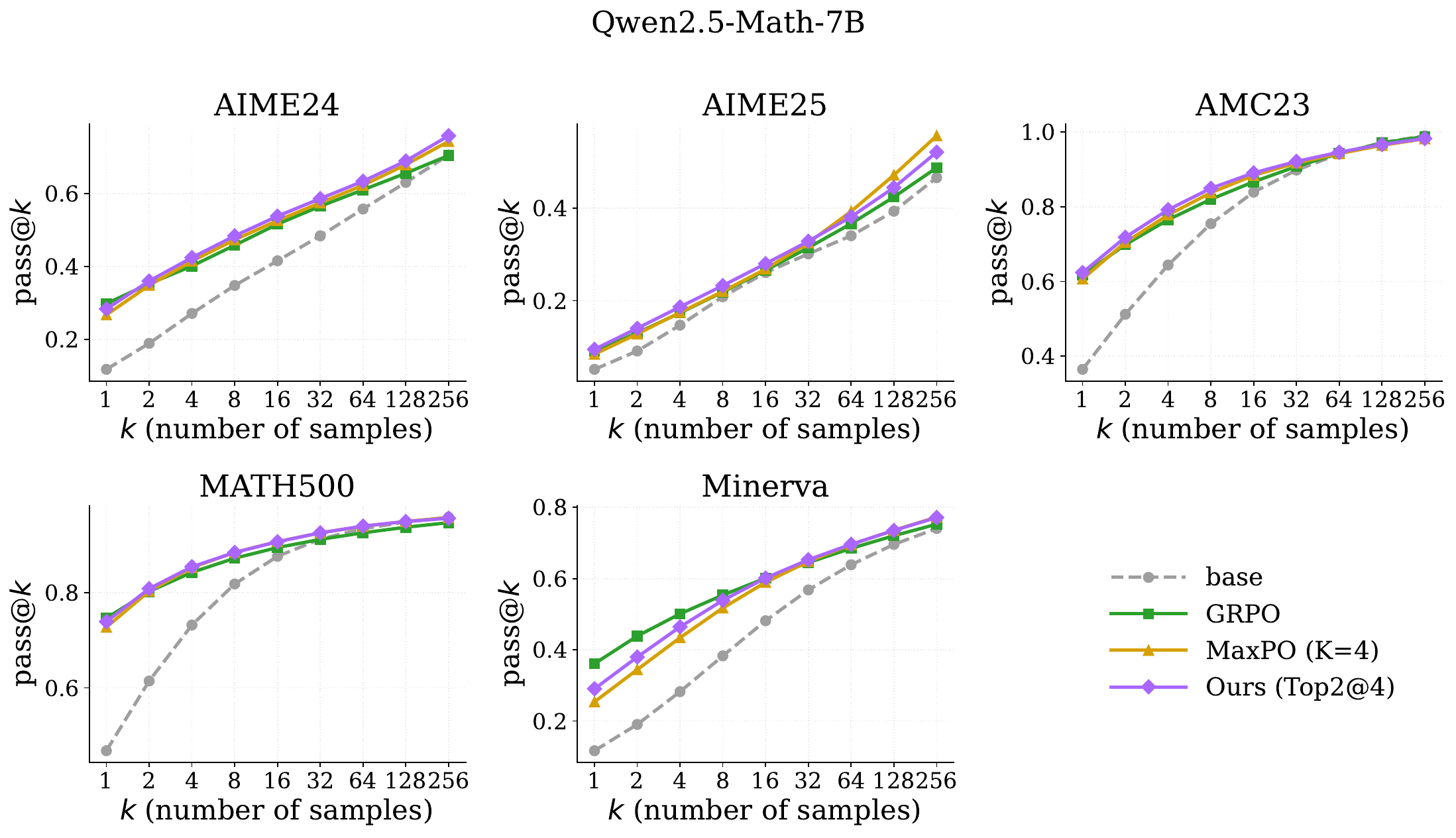}
  \caption{Per-task pass@$k$ on Qwen2.5-Math-7B.}
  \label{fig:per_task_qwen}
\end{figure}

\begin{figure}[H]
  \centering
  \includegraphics[width=\linewidth]{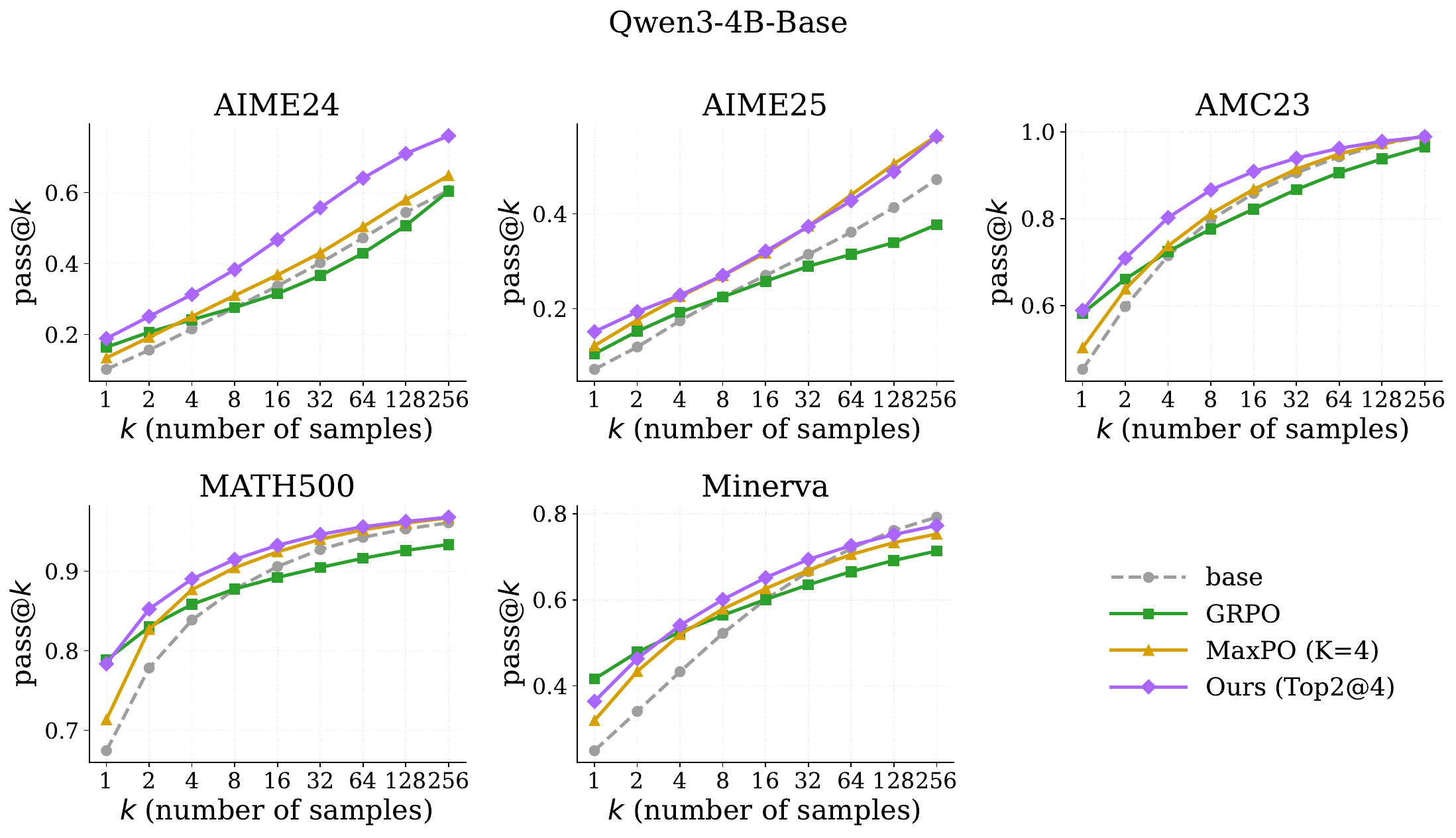}
  \caption{Per-task pass@$k$ on Qwen3-4B-Base.}
  \label{fig:per_task_qwen3}
\end{figure}

Table~\ref{tab:length_qwen} shows the pass@$k$ results for the length penalty on Qwen2.5-Math-7B. Figure~\ref{fig:length_qwen} provides the per-benchmark pass@$k$ curves.

\begin{figure}[H]
  \centering
  \includegraphics[width=\linewidth]{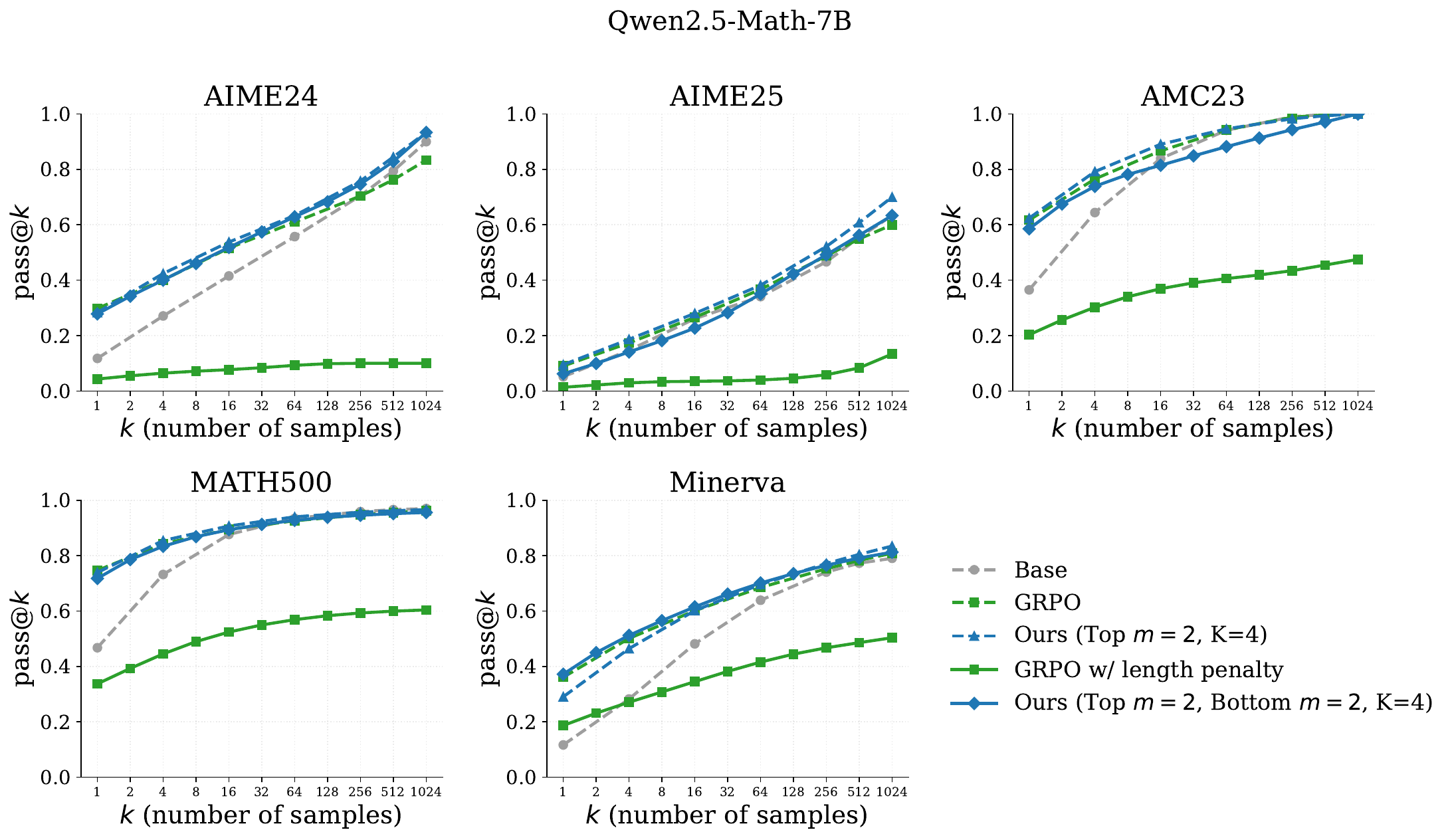}
  \caption{Per-benchmark pass@$k$ curves on Qwen2.5-Math-7B with response length penalty. Top $m=2$ by correctness reward; Bottom $m=2$ by response-length reward.}
  \label{fig:length_qwen}
\end{figure}

\newpage

\section{RL Experiments}
In this appendix, we describe an on-policy RL algorithm based on OrderGrad in the known-$(\bar R,p)$ regime explained in App.~\ref{app:known-distribution-details}, and report experimental results with MinAtar \citep{young2019minatar}.
The experiment closely follows the setup of \citet{nishimori2026emergence}.

\paragraph{PPO surrogate.}
PPO collects trajectories under the behavior policy $\pi_{\mathrm{old}}$, computes $\lambda$-returns $R^\lambda_t$, and forms advantages $A(t)=R^\lambda_t - V_\phi(s_t)$.
It then optimizes a clipped, importance-weighted surrogate with
$r_\theta(t)=\pi_\theta(a_t\mid s_t) / \pi_{\mathrm{old}}(a_t\mid s_t)$ and clipping parameter $\epsilon>0$:
\begin{equation}
\begin{aligned}
\mathcal{L}_{\scriptscriptstyle\text{PPO}}(\theta)
:=
\E{\min{\left(
r_\theta(t)A(t),\
\mathrm{clip}\big(r_\theta(t),1{-}\epsilon,1{+}\epsilon\big)A(t)
\right)}} .
\end{aligned}
\label{eq:ppo-returns}
\end{equation}
OrderGrad PPO uses the same surrogate, but replaces the standard value-based advantage with the order-statistic advantage described below.

\paragraph{OrderGrad PPO.}
We instantiate OrderGrad in a PPO-style algorithm, which we call OrderGrad PPO.
The main difference from standard PPO is that the policy advantage is computed from an order-statistic objective using a Q-critic $Q_\phi$, where $\phi$ denotes the parameters of the Q-network, rather than from the usual value-based advantage estimate used as a proxy for the return.
Specifically, for each state-action pair $(s_t,a_t)$, we first form the vector $q_t = Q_\phi(s_t,\cdot) \in \mathbb{R}^{|\mathcal A|}$ over the action set $\mathcal A$ and the policy vector $\pi_t = \pi_\theta(s_t) \in \Delta^{|\mathcal A|}$.
We then construct $\tilde q_t$ by replacing the entry of $q_t$ corresponding to the sampled action $a_t$ with the trajectory return $R^\lambda_t$, following \citet{nishimori2026emergence}, to fully use the on-policy return information.
The advantage is computed following Eq.~\ref{eq:app-adv-known} with $\tilde q_t$ and $\pi_t$ for the chosen order-statistic objective, such as Top-$M$@$K$ or Bottom-$M$@$K$.
This is a practical actor-critic instantiation rather than a direct implementation of the exact finite-action estimator: it uses a learned Q-critic, substitutes a sampled return into one Q entry, and optimizes a clipped PPO surrogate.
Algorithm~\ref{alg:ordergrad-ppo-5line} summarizes the resulting procedure.
At each iteration, we collect trajectories with the current policy, compute $\lambda$-returns, estimate order-statistic improvement scores from the Q-critic $Q_\phi$ and policy $\pi_\theta$, and use the resulting advantages in the PPO surrogate in Eq.~\ref{eq:ppo-returns}.

\begin{algorithm}[t]
  \caption{OrderGrad PPO}
  \label{alg:ordergrad-ppo-5line}
  \begin{algorithmic}[1]
    \REPEAT
      \STATE Collect trajectories under $\pi_\theta$ and compute returns $R^\lambda_t$.
      \STATE For each $(s_t,a_t)$, form $q_t \leftarrow Q_\phi(s_t,\cdot)$ and $\pi_t \leftarrow \pi_\theta(s_t)$.
      \STATE Construct $\tilde q_t$ by replacing the entry of $q_t$ corresponding to $a_t$ with $R^\lambda_t$.
      \STATE Compute the advantage $A(t)$ following Eq.~\ref{eq:app-adv-known} with $\tilde q_t$ and $\pi_t$ for the chosen order-statistic objective.
      \STATE Update the actor by the PPO surrogate in Eq.~\ref{eq:ppo-returns} with $A(t)$; update the critic $Q_\phi$ toward $R^\lambda_t$.
    \UNTIL convergence
  \end{algorithmic}
\end{algorithm}

\paragraph{Setup.}
We evaluate this algorithm on MinAtar environments implemented in pgx~\citep{young2019minatar,koyamada2023pgx}, using Asterix, Breakout, and Space Invaders, following \citet{nishimori2026emergence}.
The goal of this experiment is to test whether OrderGrad can modulate exploration behavior through its order-statistic parameters, without relying on explicit entropy regularization.
We use the Top-$M$@$K$ objective as a representative order-statistic objective of OrderGrad, fix $K=10$, and vary $M$.
We compare against PPO-V and PPO-Q, both with and without entropy regularization.
All methods are trained for $10^7$ environment steps over 10 random seeds.
PPO hyperparameters mostly follow the pgx defaults, while OrderGrad PPO uses the same setup except that the GAE parameter $\lambda$ is tuned separately.
We report the median, IQM, and mean of the normalized score following \citet{agarwal2021deep}.
The maximum scores used for normalization are taken from \citet{nishimori2026emergence}: 251.15 for Breakout, 64.95 for Asterix, and 880.91 for Space Invaders.

\paragraph{Network architecture.}
We use the same network architecture as the public implementation of PPO~\citep{koyamada2023pgx}.
For PPO-based methods, the network first maps the observation to a latent state using a shared CNN with a $2{\times}2$ convolution, ReLU activation, and average pooling, followed by an MLP.
The latent representation then branches into an actor head and a critic head.
The actor head consists of two hidden layers with ReLU and Tanh activations and outputs action logits.
For OrderGrad PPO and PPO-Q, the critic head consists of two hidden layers and outputs per-action Q-values, yielding a Q-critic $Q_\phi(s,\cdot)$.
For PPO-V, we use the corresponding scalar value critic.

\paragraph{Results.}
Figure~\ref{fig:minatar_main} shows the aggregate MinAtar performance computed with rliable~\citep{agarwal2021deep}.
In these runs, OrderGrad PPO with $M=9$ outperforms the PPO baselines in terms of normalized score.
We further analyze the effect of the parameter $M$ in Fig.~\ref{fig:minatar_m_ablation}.
The best aggregate performance is obtained around $M=9$, suggesting that the order-statistic parameter controls not only the objective but also the practical exploration--exploitation behavior of the learned policy.
This interpretation is supported by the entropy curves in Fig.~\ref{fig:minatar_entropy}, where all OrderGrad runs use entropy coefficient $0.0$.
Smaller values of $M$ slow entropy decay and maintain exploration for longer, whereas larger values of $M$ lead to faster entropy decay.
Since this variant uses critic approximation and PPO clipping, these results should be read as an empirical instantiation of the OrderGrad idea rather than a separate unbiasedness result.
Overall, the results support the flexibility of OrderGrad when combined with practical deep policy-gradient algorithms.

\begin{figure}[t]
    \centering
    \includegraphics[width=1.0\linewidth]{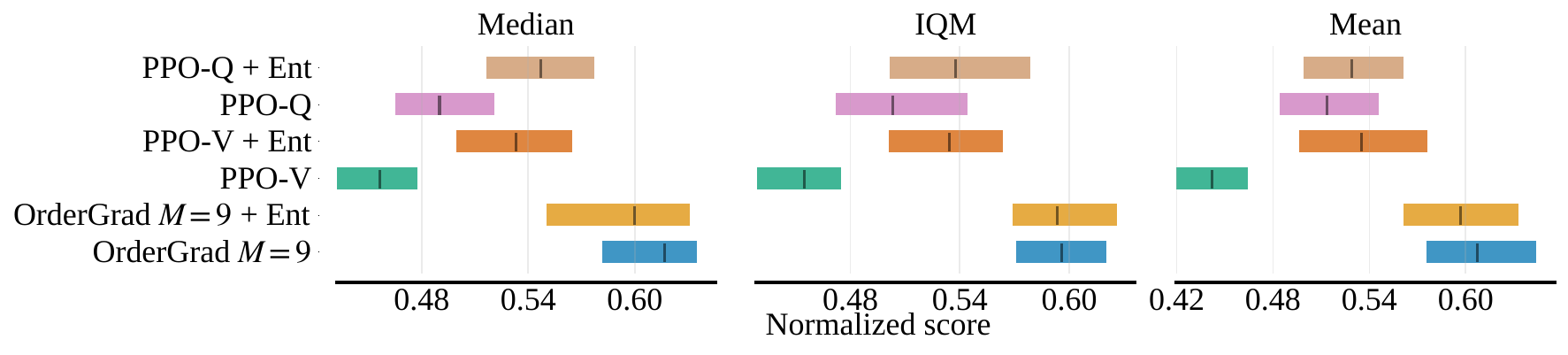}
    \caption{Aggregate MinAtar performance.
    We compare OrderGrad PPO with $M=9$ against PPO-V and PPO-Q, with and without entropy regularization.
    Scores are normalized and aggregated across environments.}
    \label{fig:minatar_main}
\end{figure}

\begin{figure}[t]
    \centering
    \includegraphics[width=0.95\linewidth]{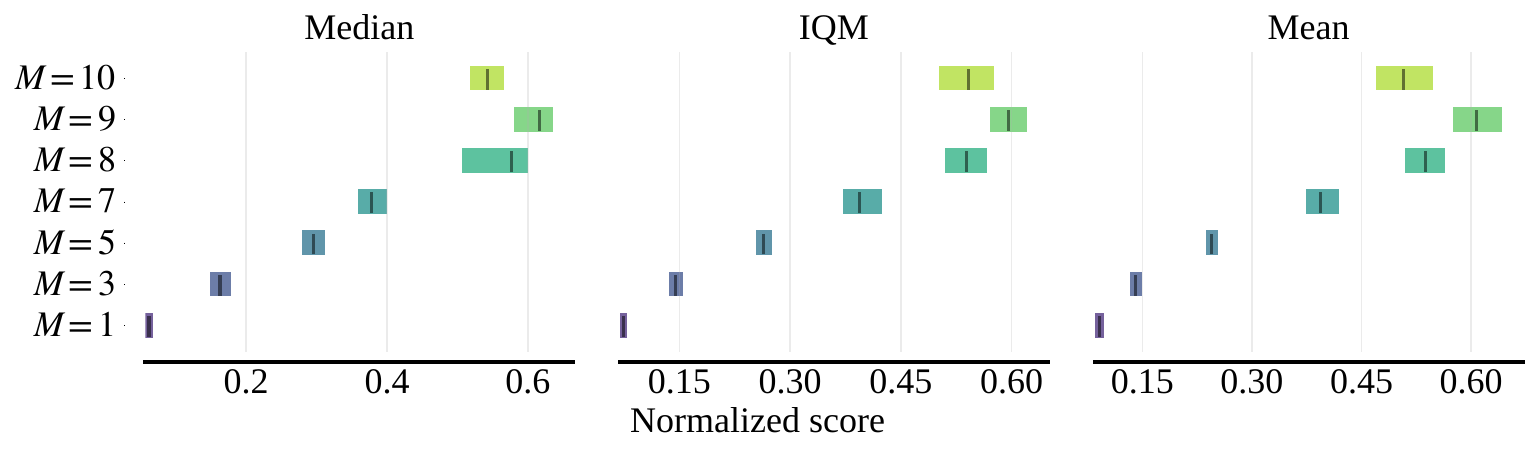}
    \caption{Effect of $M$ on MinAtar without entropy regularization.
    We report aggregate normalized evaluation return across games.
    All OrderGrad curves use entropy coefficient $0.0$.
    The best performance occurs around $M=9$.}
    \label{fig:minatar_m_ablation}
\end{figure}

\begin{figure}[t]
    \centering
    \includegraphics[width=1.0\linewidth]{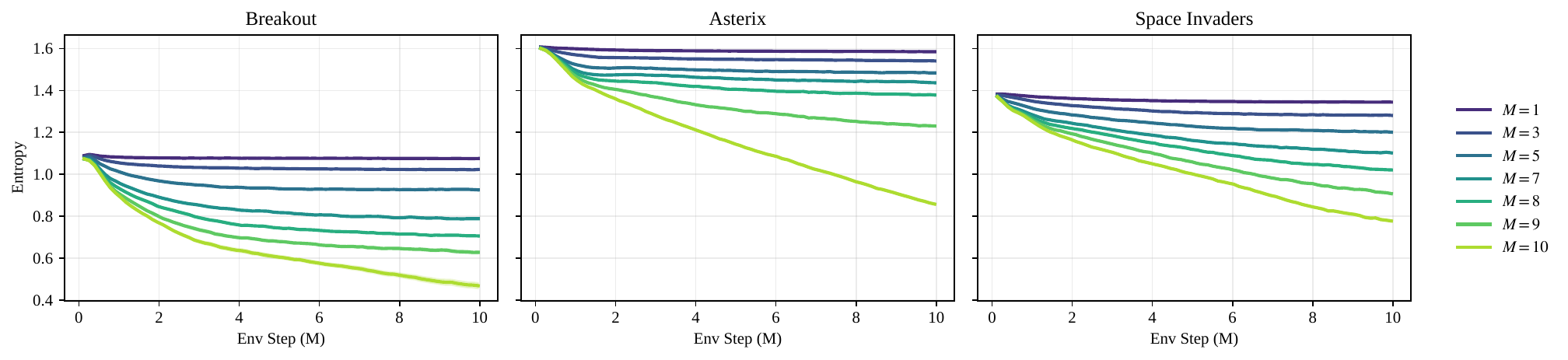}
    \caption{Policy entropy under different values of $M$ without entropy regularization.
    All curves use entropy coefficient $0.0$.
    Smaller $M$ slows entropy decay and encourages more persistent exploration,
    whereas larger $M$ leads to faster entropy decay.}
    \label{fig:minatar_entropy}
\end{figure}

\begin{table*}[t]
    \centering
    \begin{minipage}[t]{0.48\linewidth}
        \centering
        \caption{Hyperparameters for PPO-V/PPO-Q.}
        \label{tab:ppo-hparams}
        \small
        \begin{tabular}{ll}
        \toprule
        \textbf{Name (symbol)} & \textbf{Value} \\
        \midrule
        total timesteps & $1.0\times 10^{7}$ \\
        learning rate  & $3.0\times 10^{-4}$ \\
        rollout length  & 128 \\
        parallel envs  & 1024 \\
        update epochs  & 3 \\
        minibatch size  & 1024 \\
        discount $\gamma$ & 0.99 \\
        GAE $\lambda$  & 0.95 \\
        clip $\epsilon$  & 0.2 \\
        entropy coefficient  & $(0.0, 0.01)$ \\
        value loss coefficient  & 0.5 \\
        max grad norm  & 0.5 \\
        optimizer & Adam with global-norm clip \\
        \bottomrule
        \end{tabular}
    \end{minipage}
    \hfill
    \begin{minipage}[t]{0.48\linewidth}
        \centering
        \caption{Hyperparameters for OrderGrad PPO.}
        \label{tab:reppo-hparams}
        \small
        \begin{tabular}{ll}
        \toprule
        \textbf{Name (symbol)} & \textbf{Value} \\
        \midrule
        total timesteps & $1.0\times 10^{7}$ \\
        learning rate  & $3.0\times 10^{-4}$ \\
        rollout length  & 128 \\
        parallel envs  & 1024 \\
        update epochs  & 3 \\
        minibatch size  & 1024 \\
        discount $\gamma$ & 0.99 \\
        GAE $\lambda$  & 0.8 \\
        clip $\epsilon$  & 0.2 \\
        entropy coefficient  & $(0.0, 0.01)$ \\
        value loss coefficient  & 0.5 \\
        max grad norm  & 0.5 \\
        optimizer & Adam with global-norm clip \\
        top-$M$ draws $M$ & $(1,3,5,7,8,9,10)$ \\
        \bottomrule
        \end{tabular}
    \end{minipage}
\end{table*}

\end{document}